\pdfoutput=1

\documentclass[11pt]{article}

\usepackage[final]{emnlp2024/acl}

\usepackage{times}
\usepackage{latexsym}

\usepackage[T1]{fontenc}

\usepackage[utf8]{inputenc}

\usepackage{microtype}

\usepackage{inconsolata}

\usepackage{graphicx}
\usepackage{multirow}
\usepackage{booktabs}
\usepackage{amsfonts}
\usepackage{amsmath}
\usepackage{bm}
\usepackage{subcaption}
\usepackage{tcolorbox}
\usepackage{xcolor} %

\makeatletter
\newcommand\smaller{\@setfontsize\smaller{8}{9.5}}
\makeatother

\usepackage{listings}
\usepackage{color} %
\definecolor{codegreen}{rgb}{0,0.6,0}
\definecolor{codegray}{rgb}{0.5,0.5,0.5}
\definecolor{codepurple}{rgb}{0.58,0,0.82}
\definecolor{backcolour}{rgb}{0.95,0.95,0.92}

\lstdefinestyle{mystyle}{
  commentstyle=\color{blue},
  keywordstyle=\color{black},
  basicstyle=\ttfamily\fontsize{8}{9.5}\selectfont,  %
  breakatwhitespace=true,         
  breaklines=true,                 
  captionpos=b,                    
  keepspaces=true,                 
  numbersep=3pt,                  
  showspaces=false,                
  showstringspaces=false,
  showtabs=false,                  
  tabsize=2,
  breakindent=0pt,
  numbersep=3pt,
  numbers=left,
  numberstyle=\tiny\ttfamily\color{codegray},
  aboveskip=-5pt,
  belowskip=-5pt,
}
\usepackage{algorithm, algcompatible}
\usepackage[noend]{algpseudocode}
\usepackage{setspace}
\usepackage{amsthm}
\usepackage{comment}

\newcommand{\bp}{\mathbf{p}}

\newcommand{\bx}{\mathbf{x}}

\newcommand{\bz}{\mathbf{z}}

\newcommand{\CD}{\mathcal{D}}

\newcommand{\CS}{\mathcal{S}}

\newcommand{\CU}{\mathcal{U}}
\newcommand{\CX}{\mathcal{X}}

\providecommand{\norm}[1]{\left\lVert#1\right\rVert}

\newcommand{\bbR}{\mathbb{R}}

\DeclareMathOperator{\softmax}{softmax}
\DeclareMathOperator{\clip}{clip}
\DeclareMathOperator{\logits}{logits}
\DeclareMathOperator{\public}{public}

\newcommand{\hd}{\hat{d}}

\newcommand{\htheta}{\hat{\theta}}

\newcommand{\tz}{\tilde{z}}

\newcommand{\tbz}{\mathbf{\tz}}

\newcommand{\bbaz}{\bar{\bz}}
\newcommand{\bbap}{\bar{\bp}}

\newcommand{\eps}{\varepsilon}

\newtheorem{thm}{Theorem}

\newtheorem{lem}{Lemma}
\newtheorem{defn}{Definition}
\newtheorem{assum}{Assumption}



\title{Private prediction for large-scale synthetic text generation}

\author{First Author \\
  Affiliation / Address line 1 \\
  Affiliation / Address line 2 \\
  Affiliation / Address line 3 \\
  \texttt{email@domain} \\\And
  Second Author \\
  Affiliation / Address line 1 \\
  Affiliation / Address line 2 \\
  Affiliation / Address line 3 \\
  \texttt{email@domain} \\}

  \title{Private prediction for large-scale synthetic text generation\thanks{Authors ordered alphabetically. Author contributions are listed at the end.}}

\author{
    Kareem Amin \quad Alex Bie \quad Weiwei Kong \quad Alexey Kurakin \\ Natalia Ponomareva \quad Umar Syed \quad Andreas Terzis \quad Sergei Vassilvitskii \\
    Google \\
    {\small \texttt{\{kamin, alexbie, weiweikong, kurakin, nponomareva, usyed, aterzis, sergeiv\}@google.com}}
 }

\author{
 \textbf{Kareem Amin\quad Alex Bie\quad Weiwei Kong\quad Alexey Kurakin} \\
 \textbf{Natalia Ponomareva\quad Umar Syed\quad Andreas Terzis\quad Sergei Vassilvitskii} \\
 Google \\
 {\small\texttt{\{kamin,alexbie,weiweikong,kurakin,nponomareva,usyed,aterzis,sergeiv\}@google.com}}
}

\begin{document}
\maketitle
\begin{abstract}
We present an approach for generating differentially private synthetic text using large language models (LLMs), via private prediction. In the private prediction framework, we only require the output synthetic data to satisfy differential privacy guarantees. This is in contrast to approaches that train a generative model on potentially sensitive user-supplied source data and seek to ensure the model itself is safe to release.
We prompt a pretrained LLM with source data, but ensure that next-token predictions are made with differential privacy guarantees. Previous work in this paradigm reported generating a small number of examples ($<$10) at reasonable privacy levels, an amount of data that is useful only for downstream in-context learning or prompting. In contrast, we make changes that allow us to generate thousands of high-quality synthetic data points, greatly expanding the set of potential applications. Our improvements come from an improved privacy analysis and a better private selection mechanism, which makes use of the equivalence between the softmax layer for sampling tokens in LLMs and the exponential mechanism. Furthermore, we introduce a novel use of public predictions via the sparse vector technique, in which we do not pay privacy costs for tokens that are predictable without sensitive data; we find this to be particularly effective for structured data.
\end{abstract}

\section{Introduction}

Differentially private mechanisms process a source dataset potentially containing sensitive user information and perform a useful computation --- as simple as calculating a mean, or as complex as training an ML model --- whose output can be safely shared while protecting the privacy of users who contributed to the dataset.

Perhaps the most general-purpose differentially private mechanism is one that produces a synthetic version of its input dataset, as the output of such a mechanism would be suitable for all the same purposes as the original dataset. For example, a private synthetic dataset can be used to train an ML model, but can also be used for auxiliary tasks such as feature engineering, hyperparameter tuning, and quality monitoring.

There has been recent interest in using large-language models (LLMs) to generate differentially private versions of text datasets. Existing approaches can be classified into several categories. \emph{Private fine-tuning} methods privately adjust the parameters of an LLM on the input dataset, using an algorithm such as differentially private stochastic gradient descent (DP-SGD), and then prompt the LLM to generate similar text. Fine-tuning methods have been used to produce high-quality synthetic data, but the training procedure can be prohibitive, available only to those with the time, compute, and access necessary to train state-of-the-art LLMs containing billions of parameters.

\emph{Private prediction} methods do not modify the LLM parameters at all. Instead, they directly prompt the LLM with text from the source dataset, asking for similar text in response, and then perturb the LLM's token distribution (\emph{i.e.}, its last layer) to ensure that each sampled token, and thus the entire generated response, is private. Since no training is required, private prediction methods can quickly generate synthetic data, typically producing some data within minutes instead of hours, which allows for rapid prototyping and iteration. However, unlike private fine-tuning, the guarantees of private prediction methods degrade with the volume of data that is generated. Consequently, existing private prediction methods have mostly been used in applications that require only small amounts of synthetic data \citep{tang2024privacypreserving}, sharply limiting their practicality.

\begin{figure*}[t]
    \centering
    \includegraphics[width=0.84\textwidth,trim={0cm 0.23cm 0cm 0.23cm},clip]{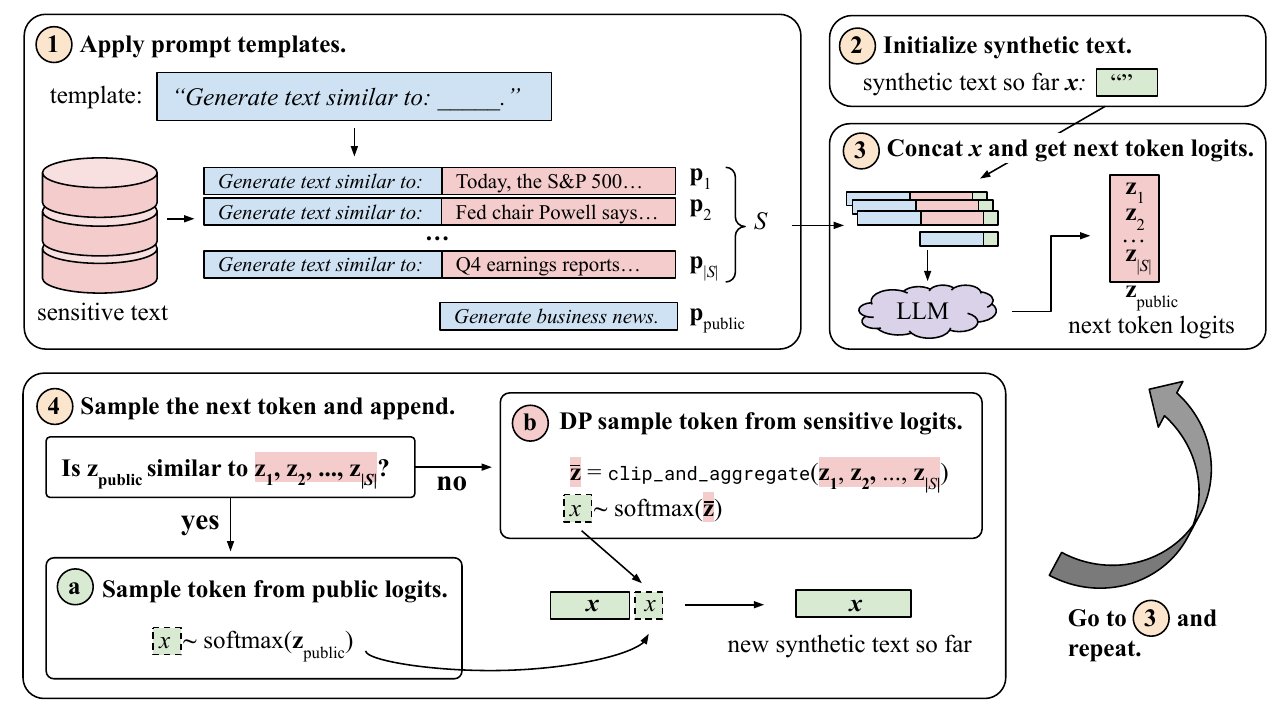}

    \caption{\small Algorithm \ref{alg:main}, visualized. An LLM receives a batch of prompts, each instructing to generate text similar to a piece of sensitive text.  \textit{Synthetic text} is generated token by token, by running inference on the batch in parallel. In each step, the logit vectors produced downstream of sensitive text are aggregated and sampled from with differential privacy. Every token sampled in such way incurs a privacy cost, motivating us to include an auxillary public prompt and sample from its logits when similar to the sensitive logits.}
    \label{fig:diagram}
\end{figure*}

\subsection{Our contributions}

In this paper we describe a new private prediction method that produces hundreds of times as much synthetic data as a state-of-the-art private prediction method, while maintaining a comparable privacy guarantee. Similar to some existing work, our method is based on running LLM inference on several subsets of the input data in parallel and privately aggregating their token distributions to generate synthetic text. However, our approach is distinguished by three novel algorithmic elements that lead to its improved performance:

\paragraph{1. Better private token selection.} Instead of protecting the privacy of the entire token distribution with Gaussian or Laplace noise, we leverage the uncertainty inherent in sampling to ensure privacy. We clip and aggregate token logits before standard softmax sampling --- which is differentially private, since it can be viewed as the exponential mechanism. Our approach induces much less distortion of the original token distributions to achieve the same level of privacy than prior work.

\paragraph{2. Avoiding prefix re-sampling.} Prior work generated each token using a random subset of the input data, leveraging privacy amplification by subsampling in their analysis. This is computationally undesirable, as it requires repeated re-computation of the prefix for each decoding step, thus limiting scalability towards generating large synthetic corpora. Indeed, prior work describes this re-sampling as the ``main weakness'' of the approach \cite{tang2024privacypreserving}. To resolve the problem, we instead generate each synthetic example using a fixed disjoint subset of the input data, which yields substantial savings in privacy cost -- via \emph{parallel composition} -- while allowing us to pay linear instead of quadratic non-attention FLOPs in terms of sequence length via KV cache accelerated decoding.

\paragraph{3. Leveraging public predictions.} Our method uses an auxillary token distribution from an LLM without access to sensitive data, and draws the next token from that distribution whenever it is very similar to the token distribution induced by the sensitive data. Our method incurs no privacy cost when outputting “obvious” tokens, and as a result, only a fraction of the tokens in the synthetic data are generated using sensitive data (as little as 20\% in structured datasets). We leverage the sparse vector technique to privately calculate distributional similarity.

\paragraph{} Taken together, the combination of these algorithmic techniques leads to significant improvements over prior work. Roughly speaking, (1) and (2) above keep our inference  closely aligned to standard (non-DP) inference. 

In our experiments, we generate private synthetic versions of publicly available, benchmark machine learning datasets, and then use the synthetic datasets for downstream classification and extraction tasks. Owing to the increased quantity and quality of our synthetic data, we improve over an existing state-of-the-art private prediction method in terms of downstream accuracy. 
Furthermore, while prior work in this paradigm only generated a small (<10) number of examples, we demonstrate the ability to generate thousands of training examples, enough for fine-tuning downstream models.

Finally, since synthetic data is intended for a wide variety of applications, we also evaluate data quality using a metric that is orthogonal to performance on downstream tasks. Specifically, we generate synthetic versions of a publicly available dataset containing highly structured data records, each of which is encoded as a JSON object. Our results demonstrate that the sparse vector technique helps preserve data structure at high privacy levels.

\section{Related work}

\emph{Private fine-tuning} is widely used for synthetic text generation. \citet{yue-etal-2023-synthetic} created private synthetic versions of text datasets by using DP-SGD \citep{abadi2016deep} to fine-tune an LLM on the sensitive data. \citet{kurakin2024harnessing} showed that parameter efficient approaches to fine-tuning, such as LoRA \citep{hu2022lora} can improve the quality of the synthetic data. \citet{wu2024prompt} took a two-stage approach: First they fine-tuned an LLM on a public dataset that closely resembled the sensitive data (which was itself generated by an LLM using carefully designed prompts); then they completed the fine-tuning process by running DP-SGD on the sensitive dataset. Concurrent to the present work, \citet{tran2024differentially} describe a private fine-tuning approach for generating synthetic tabular data that is formatting compliant.

\emph{Private prediction} \citep{pmlr-v75-dwork18a} is an alternate approach to private machine learning that only guarantees the privacy of the predictions output by an ML model, and not the model itself. The predominant way this is realized is via \emph{subsample-and-aggregate} \citep{nrs07}: First sensitive data is split into disjoint partitions; then non-private predictions are made from each partition and privately aggregated. PATE \citep{papernot2017semi, papernot2018scalable} employs this approach to get answers to a limited set of image classification queries, which are then used to train a student model that can be queried indefinitely. 

Private prediction has been applied to synthetic text generation by viewing each token sampled by an LLM as a `prediction', and perturbing the LLM's token distributions to ensure their privacy.  \citet{tang2024privacypreserving} added noise to several independent token distributions and averaged them, while \citet{hong2024dpopt} selected the most popular token among the token distributions using the LimitedDomain mechanism \citep{durfee2019practical}. These methods can avoid the time, compute, and access required to fine-tune an LLM with billions of parameters. However, a privacy loss is suffered for each \emph{token} produced in this manner. As a result, previous work has only been able to generate a very small number of synthetic examples at reasonable privacy levels (fewer than 10). Other work has applied private prediction techniques to LLMs \citep{ Majmudar2022, NEURIPS2023_f26119b4}, including in combination with fine-tuning \citep{ginart2022submix,flemings2024differentially}, but not for the purpose of synthetic text generation.

Finally, another distinct set of approaches are \emph{private filtering} methods. Private filtering methods operate directly on whole LLM responses and a large corpus of public data that does not require protection. \citet{yu2024privacypreserving} and \citet{xie2024differentially} used the sensitive responses to privately select similar responses from the public dataset. Similarly, \citet{wu2024privacypreserving} aggregate response embeddings and select the public response that is closest in embedding space.\footnote{\citet{wu2024privacypreserving} also proposes a non-filtering approach based on privately selecting common keywords among the sensitive data and using them to prompt an LLM.} One limitation of filtering methods is that the menu of possible responses is constructed without signal from the new source dataset.

\section{Method}

\subsection{Standard LLM inference} 

Before describing our algorithm for generating private synthetic text, we review the standard algorithm for LLM inference. Let $\CX$ be the token vocabulary (\emph{i.e.}, the set of all possible tokens), and let $v = |\CX|$ be the vocabulary size. A \emph{token sequence} is an element of $\CX^*$, and a \emph{logit vector} is an element of $\bbR^v$ (one logit per token in the vocabulary). If $\bx_1$ and $\bx_2$ are token sequences then we write $\bx_1\bx_2 \in \CX^*$ to denote their concatenation.
   
A decoder-only LLM can be viewed as a function $\logits : \CX^* \rightarrow \bbR^v$ that maps each token sequence to a logit vector. Standard LLM inference generates a \emph{response} $\bx \in \CX^*$ by initializing $\bx = \bp$, where $\bp \in \CX^*$ is the \emph{prompt}, and then repeatedly executes the following procedure: (1) Let $\bz = \logits(\bx)$; (2) draw token $x$ from $\softmax(\bz / \tau)$; and (3) append $x$ to $\bx$. Here $\softmax(\bz/\tau)$ is the distribution that assigns probability proportional to $\exp(z_i/\tau)$ to the $i$th token, and $\tau > 0$ is a \emph{temperature} parameter that flattens or sharpens the distribution. The procedure terminates when $x = \texttt{<eos>}$, a special token that indicates the end of the response.

\subsection{Our algorithm} 
One straightforward approach to generating a synthetic version of a sensitive piece of text would be to prompt an LLM with `{\it Please generate text similar to: <sensitive text>}'. However, this could easily lead to a privacy violation, as the response could retain the semantics of the input sensitive text. 

Algorithm \ref{alg:main} describes our method for privately generating a dataset of synthetic examples $X$ from a dataset of sensitive prompts $D$. Each prompt in $D$ resembles the sample prompt given above. But instead of using a \emph{single} prompt to generate a synthetic example, Algorithm \ref{alg:main} takes a \emph{batch} of prompts $S$ and runs LLM inference in parallel on each prompt. A synthetic example is generated one token at a time, with the average of the logit vectors across the batch defining the distribution from which the next token is randomly selected. Before averaging, a logit vector $\bz$ is clipped and re-centered using the function
\begin{equation}
    \clip_c(\bz)_i = \max\{-c, \bz_i - \max_j {\{\bz_j\}} + c\} \label{eq:clip}
\end{equation}
which maps each component $i$ of $\bz$ into the target clipping range $[-c, c]$. Forcing each logit to lie in a bounded range is key to proving the privacy guarantee for our algorithm (see \S\ref{sec:analysis}). While several functions can achieve this purpose, Eq.~\eqref{eq:clip} has an additional desirable property: If the components of $\bz$ can be shifted by a constant so that they all lie in the interval $[-c, c]$, then $\clip_c(\bz)$ is one such shift. This property is desirable because the distribution $\softmax(\bz)$ is invariant to any constant shift of $\bz$. Empirically, we found that Eq.~\eqref{eq:clip} performed better than other functions considered. For example, regular clipping to the range $[-c,c]$ without recentering requires twice as large $c$ to sample without distortion (see \S\ref{appendix:design-choices:clipping}).

\begin{center}
\begin{algorithm}[!t]
\small
\setstretch{1.15}
\caption{\label{alg:main} Generate private synthetic examples}
\begin{algorithmic}[1]
\Statex {\bf Parameters:} LLM $\logits(\cdot)$, public prompt $\bp_{\public}$, expected batch size $s$, private tokens to sample $r$. \emph{Sampling}: clipping bound $c$, temperature $\tau$, public temperature $\tau_{\public}$. \emph{Public optionality}: public/private distribution distance $\mathrm{d}(\cdot,\cdot)$, threshold $\theta$, noise level $\sigma$.   
\Statex {\bf Input:} Dataset of sensitive prompts $D$; each prompt contains a sensitive example
\Statex {\bf Output:} Dataset of synthetic examples $X$
\State $X \gets \emptyset$
\State Let $\CS$ be a partition of $D$ into disjoint batches
\For{each batch $S \in \CS$}
\State $\htheta \gets \theta + \textrm{Laplace}(\sigma)$\quad{\smaller\color{blue}\texttt{\# Init noisy threshold}}
\State $t \gets 0$\quad{\smaller\color{blue}\texttt{\# Private token counter}}
\While{$t < r$}
\State $\bx \gets \textrm{Empty token sequence}$
\While{$\bx$ does not end with \texttt{<eos>}}
\State $Z \gets \{\logits(\bp\bx) : \bp \in S\}$
\State $\bz_{\public} \gets \logits(\bp_{\public}\bx)$

\qquad\hspace{6pt}{\smaller\color{blue}\texttt{\# Check if pub/priv distributions are far}}
\State $\hd \gets \mathrm{d}(Z, \bz_{\public}) + \textrm{Laplace}(2\sigma)$
\If{$\hd \geq \htheta$}\quad{\smaller\color{blue}\texttt{\# Sample priv token}}
\State $\bbaz \gets \frac1s \sum_{\bz \in Z} \clip_c(\bz)$
\State $x \sim \softmax(\bbaz / \tau)$
\State $t \gets t + 1$
\State $\htheta \gets \theta + \textrm{Laplace}(\sigma)$
\Else\quad{\color{blue}\smaller\texttt{\# Sample pub token}}
\State $x \sim \softmax(\bz_{\public} / \tau_{\public})$ 
\EndIf
\State Append $x$ to $\bx$
\EndWhile
\State $X \gets X \cup \{\bx\}$
\EndWhile
\EndFor
\State \textbf{return} $X$
\end{algorithmic}
\end{algorithm}
\end{center}

Since the average logit vector is computed using sensitive prompts, each token selected from a distribution determined by the average logit vector incurs a privacy cost. To minimize this cost, Algorithm \ref{alg:main} also has access to a non-sensitive public prompt, $\bp_{\public}$, and uses this prompt to generate the next token whenever doing so does not significantly change the distribution from which the next token is drawn. The distance function used to make this determination is 
\begin{equation}
    \mathrm{d}(Z, \bz_{\public}) = \|\frac 1 s \sum_{\bz \in Z} p_\bz - p_{\bz_{\public}}\|_1, \label{eq:distance}
\end{equation}
where $p_\bz:=\softmax(\bz)$, $Z$ are the logit vectors computed for each sensitive prompt in $S$, $\bz_{\public}$ is the logit vector computed using $\bp_{\public}$, and $s$ is the expected batch size. When this distance is small, Algorithm \ref{alg:main} outputs a public token instead of a private token. The privacy guarantee for Algorithm \ref{alg:main} leverages the analysis of the sparse vector technique \citep{dwork2009complexity}, and shows that while privacy degrades with the number of private output tokens, it is independent of the number of public output tokens (see \S\ref{sec:analysis}). Empirically, we observe that the fraction of output tokens that must be private in order to generate high-quality synthetic data can be as low as 20\% for highly structured datasets.

Note that the first step of Algorithm \ref{alg:main} partitions the input dataset of sensitive prompts into disjoint batches. We do not prescribe a procedure for assigning prompts to batches in Algorithm \ref{alg:main} since many batching approaches are admissible as long as they satisfy a minor technical assumption required for the privacy analysis of Algorithm \ref{alg:main}, which we explain in \S\ref{sec:analysis}.  
While the batches are not required to be any particular size, the algorithm runs faster if each batch has size equal to the expected batch size $s$. And while prompts can be batched together (almost) arbitrarily, more tailored batching can lead to better synthetic data quality. For example, in the experiments in \S\ref{sec:experiments}, where we generate synthetic versions of ML training datasets, each sensitive prompt contains a label. In those experiments we assign prompts with the same label to the same batch.

\subsection{Comparison to prior algorithms}
Two major features of Algorithm \ref{alg:main} are that it leverages the inherent randomness of token sampling to guarantee privacy, and that it further reduces privacy cost by using public data to generate a portion of the synthetic data. Some prior work also incorporated these algorithmic ideas, but with key differences. Instead of clipping logits to ensure that the token sampling is private, \citet{Majmudar2022} mixed each sensitive token distribution with the uniform distribution. This approach induced a dependence on the vocabulary size in their privacy guarantee, and since LLM vocabularies are typically very large, the resulting privacy guarantee was quite weak: \citet{Majmudar2022} noted that setting the differential privacy parameter $\varepsilon$ lower than 50 produced synthetic data that was ``unusable''. \citet{flemings2024differentially} guaranteed the privacy of token sampling by mixing each sensitive token distribution with a public token distribution, but their approach was based on aggregating a set of fine-tuned models, not a set of prompts. Neither \citet{Majmudar2022} nor \citet{flemings2024differentially} aim to generate synthetic data.

\citet{tang2024privacypreserving} found that limiting the token vocabulary to a fixed set of the most popular 100 public tokens caused their synthetic data generation algorithm to exhibit greater stability. However, if the sensitive data contains many tokens that are rare in public data, their approach cannot produce synthetic data that is very similar to the sensitive data. By contrast, our approach compares public and private token distributions on-the-fly, and determines which one to use for sampling the next token by balancing a trade-off between privacy and quality. Also, \citet{tang2024privacypreserving} used a different random subset of prompts to generate each token, and left as an open problem how to use a single subset to generate every token in a synthetic example. Our algorithm resolves this open problem, and consequently yields both improved privacy and greater computational efficiency (see \S\ref{sec:discussion}).

\section{Privacy analysis}
\label{sec:analysis}

In this section we state formally how Algorithm \ref{alg:main} preserves the privacy of the sensitive prompts it uses to generate synthetic examples. 

Let $\CD$ be the set of all possible prompt datasets. A \emph{mechanism} is a randomized function with domain $\CD$. Note that Algorithm \ref{alg:main} is a mechanism. We say that a pair of prompt datasets $D, D' \in \CD$ are \emph{neighbors} if there exists a prompt $\bp$ such that $D = D' \cup \{\bp\}$ or $D' = D \cup \{\bp\}$. In the differential privacy literature this is commonly referred to as the \emph{add/remove} neighbor relation. 

\begin{defn}[\citet{dwork2006our}] \label{defn:dp} A mechanism $M$ satisfies \emph{$(\eps, \delta)$-differential privacy} if
$\Pr[M(D) \in O] \le e^\eps \Pr[M(D') \in O] + \delta$ for any neighboring datasets $D, D' \in \CD$ and subset $O$ of the range of $M$.
\end{defn}

Theorem \ref{thm:main} below provides a differential privacy guarantee for Algorithm \ref{alg:main}. The proof of Theorem \ref{thm:main} requires a technical assumption about how the prompts are partitioned into batches in the first step of the algorithm.

\begin{assum} \label{assum:batch}
In Algorithm \ref{alg:main}, the assignment of a prompt to a batch depends only on the prompt itself, and not on the other prompts. \end{assum}

The most straightforward way to satisfy Assumption \ref{assum:batch} is to apply a hash function to each prompt and then use the hash value to determine its assigned batch. For example, if $h$ is the hash value, $n$ is the number of prompts and $s$ is the expected batch size, then we can assign the prompt to the $(h \mod \frac{n}{s})$th batch. If we want to batch together prompts that share a certain attribute (like a label), we can apply another hash function to that attribute and concatenate the hash values. Using hash functions for batch assignment can lead to batches whose sizes differ from the expected batch size $s$, but this does not impact the validity of Theorem \ref{thm:main}.

\begin{thm}[Privacy of Algorithm \ref{alg:main}] \label{thm:main} Suppose Assumption \ref{assum:batch} holds. Let $\rho = r\left(\frac12 \left(\frac{c}{s\tau}\right)^2 + \frac{2}{(s\sigma)^2}\right)$. For all $\eps \ge 0$, Algorithm \ref{alg:main} satisfies $(\eps, \delta)$-differential privacy, where 
$$\delta = \inf_{\alpha \in (1,\infty)} \frac{e^{(\alpha -1)(\alpha \rho - \varepsilon)}}{\alpha-1} \left( 1- \frac{1}{\alpha} \right)^\alpha.$$
Also, for all $\delta \in (0, 1]$, Algorithm \ref{alg:main}
satisfies $(\eps, \delta)$-differential privacy, where 
$$\eps = \rho + \sqrt{4\rho \log(1/\delta)}.$$
\end{thm}

The proof is in \S\ref{appendix:proof} and makes use of sharp privacy analyses of: (1) zCDP to approximate DP conversion \citep{canonne2020discrete}; and (2) zCDP bounds for the exponential mechanism \citep{cesar2021bounding}.

\section{Experiments}
\label{sec:experiments}

Gemma 1.1 2B IT \citep{gemmateam2024gemma} is the data generator in our main private prediction experiments. We choose it due to its lightweight, open-source JAX implementation\footnote{\url{https://github.com/google-deepmind/gemma}} that makes easy to implement and share sampling algorithms. Tables \ref{tab:datasets} and \ref{tab:models} give an overview of datasets and models used.

\begin{table}[!ht]
\centering
    \begin{subtable}{0.48\textwidth}
    \small
    \centering
    \scalebox{0.95}{
    \begin{tabular}{lrl}
    \toprule
    Dataset & $n_\text{train}$ & Description\\
    \midrule
    AGNews & 120,000 & 4-way news classification\\
    TREC & 5452 & 6-way query classification\\
    DBPedia & 560,000 & 14-way topic classification\\
    MIT-G &  2,953 & Movie genre extraction\\
    MIT-D & 1,561 &  Movie director extraction\\
    IMDB & 25,000 & 2-way review classification\\
    Yelp & 560,000 & 2-way review classification\\
    WikiMoviesJSON & 27,412 & JSON with 6 fields\\
    \bottomrule
    \end{tabular}
    }
    \caption{\small Overview of datasets used.}
    \label{tab:datasets}
    \end{subtable}
    \qquad
    \begin{subtable}{0.5\textwidth}
    \small
    \centering
    \vspace{10pt}
    \scalebox{0.95}{
    \begin{tabular}{ll}
    \toprule
    Model  & Usage\\
    \midrule
    Gemma 1.1 2B IT & Generation; private prediction \\
    LaMDA 8B & Generation; DP fine-tuning \\
    \midrule
    GPT-3 babbage-002 & Evaluation; in-context learning \\
    BERT-Base 12/768 110M  & Evaluation; fine-tuning \\
    \bottomrule
    \end{tabular}
    }
    \caption{\small Overview of models used in main experiments.}
    \label{tab:models}
    \end{subtable}
    \caption{\small Overview of datasets and models used in our main experiments. Datasets are benchmark classification and extraction tasks used in prior work on private synthetic text generation, with the exception of WikiMoviesJSON, which is used for structured data experiments. LaMDA and Gemma are used for synthetic data generation, while the other models are used to measure how useful our synthetic data is for improving accuracy on real test data.}
\end{table}

We perform 3 sets of experiments, targeting various datasets and utility criteria:

\paragraph{In-context learning (\S\ref{sec:experiments:icl}).} We generate examples to use as in-context exemplars for prompting an LLM. We report downstream accuracy on real test examples, when prompted with synthetic data, on 3 classification tasks (\emph{AGNews} \citep{zhang2015character}, \emph{DBPedia} \citep{zhang2015character}, \emph{TREC} \citep{voorhees2000building}) and 2 extraction tasks (\emph{MIT-G}, \emph{MIT-D} \citep{liu2012conversational}). 

\paragraph{Fine-tuning (\S\ref{sec:experiments:fine-tuning}).} We generate synthetic examples to use for fine-tuning a BERT classifier. We report downstream accuracy on real test examples for 3 classification tasks (\emph{IMDB} \citep{maas2011learning}, \emph{Yelp} \citep{zhang2015character}, \emph{AGNews} \citep{zhang2015character}). 

\paragraph{Structured data (\S\ref{sec:experiments:structured}).} We generate examples that must adhere to structural constraints to be useful synthetic data. We consider a JSON generation task (\emph{WikiMoviesJSON} \citep{rust2024wikipediamoviedata}), evaluating structure preservation.

\subsection{In-context learning}\label{sec:experiments:icl}

\paragraph{Experimental setup.} Using our method, we generate 90-1500 examples using Gemma 1.1 2B IT. We compare against real examples, and results reported in the prior work of \citet{tang2024privacypreserving}, where they generated 4-shot examples for in-context learning.\footnote{It is no longer possible to reproduce their results, due to changes in the OpenAI API since publication: GPT-3 babbage is now deprecated, and it is no longer possible to query for top 100 logprobs, which is required by their method.}
To evaluate generated synthetic data, we put synthetic examples in the context window before querying with the real test example, as shown in Figure \ref{fig:icl-evaluation}.

\begin{figure}[!ht]

\centering
\begin{minipage}{0.47\textwidth}
\begin{center}
\begin{tcolorbox}
\smaller
\begin{lstlisting}[language=Python]
Classify the following examples:
Text: lorem ipsum   # synthetic text 1
Answer: label
# ...
Text: sed do eiusm  # synthetic text n
Answer: label

Text: excepteur si  # test text
Answer: 
\end{lstlisting}
\end{tcolorbox}
\end{center}
\end{minipage}
\caption{\small Example of $n$-shot in-context learning evaluation for synthetic data.}
\label{fig:icl-evaluation}
\end{figure}

\begin{table*}[t!]\centering
    \scalebox{0.75}{
    \centering
    \tabcolsep=0.1cm
    \begin{tabular}{lllllrrrrr}
    \toprule
    & & & & &\multicolumn{5}{c}{GPT-3 babbage-002 Acc. (\%)\textbf{*}} \\
    \cmidrule(r){6-10}
    \multicolumn{1}{c}{$\varepsilon$} & Method & Shots & Reported in & Model &\multicolumn{1}{c}{AGNews} &\multicolumn{1}{c}{DBPedia} & \multicolumn{1}{c}{TREC} & \multicolumn{1}{c}{MIT-G} & \multicolumn{1}{c}{MIT-D}\\
    \midrule
    0 & Zero shot & 0 & This work & - & $24.8_{0.0}$ & $12.0_{0.0}$ & $28.4_{0.0}$ & $29.6_{0.0}$ & $28.8_{0.0}$  \\
    \midrule
    \multirow{6}{*}{$\infty$} & \multirow{2}{*}{Real data} & 4 & \multirow{2}{*}{This work} & \multirow{2}{*}{-} & $75.3_{3.0}$ & $73.6_{0.3}$ & $34.9_{5.0}$ & $56.0_{2.0}$ & $83.1_{5.3}$  \\
    &  & 64 & &  & $84.7_{1.5}$  & $92.5_{1.6}$& $50.3_{6.1}$ & $56.4_{5.4}$ & $89.1_{0.7}$ \\
    \cmidrule(r){2-10}
    & \multirow{1}{*}{\citet{tang2024privacypreserving}} & \multirow{1}{*}{4}& \citet{tang2024privacypreserving}* &  GPT-3 babbage & $69.3_{4.8}$ & $82.3_{3.7}$ & $50.6_{6.9}$ & $54.4_{7.0}$ & \multicolumn{1}{c}{-}\\
    \cmidrule(r){2-10}
    & \multirow{2}{*}{Ours} & 4 & \multirow{2}{*}{This work} & \multirow{2}{*}{Gemma 1.1 2B IT} & $76.8_{4.8}$ & $72.3_{2.5}$ & $38.8_{6.0}$ &  $47.7_{2.5}$& $81.7_{2.4}$  \\
    &  & 64 & &  & $77.5_{1.8}$ & $91.5_{1.7}$ & $57.9_{3.4}$ & $56.4_{1.2}$  &  $87.1_{0.2}$\\
    \midrule
    \multirow{4}{*}{$1$} & \multirow{2}{*}{\citet{tang2024privacypreserving}} & \multirow{1}{*}{4} & \citet{tang2024privacypreserving}* &  GPT-3 babbage &  $64.1_{3.9}$ & $81.2_{3.0}$ & $50.7_{4.1}$ & $46.3_{7.8}$ & $69.2_{7.9}$\\
    &  & 4 & This work & Gemma 1.1 2B IT & $74.9_{3.8}$ &  $80.9_{3.6}$& $36.7_{2.2}$ & $34.1_{9.3}$  & $78.7_{1.9}$ \\
    \cmidrule(r){2-10}
    & \multirow{2}{*}{Ours} & 4 & \multirow{2}{*}{This work} & \multirow{2}{*}{Gemma 1.1 2B IT} & $75.9_{3.5}$ & $75.1_{0.5}$ & $39.2_{3.7}$  & $47.1_{6.0}$ & $84.5_{1.0}$   \\
    &  & 64 &  &  & $78.7_{1.8}$ &  $90.4_{2.6}$& $53.6_{1.3}$ & $51.6_{2.3}$  & $86.4_{0.6}$ \\
    \bottomrule
    \end{tabular}
    }
    \caption{\small In-context learning results with GPT-3 babbage-002. We report mean and standard deviation over 3 random samplings (equally many from each label for classification; fully random for extraction) of synthetic/real data. \textbf{(*) Note}: For the results reported in \citet{tang2024privacypreserving}, they use GPT-3 babbage (now deprecated; we use GPT-3 babbage-002) as the downstream in-context learner, and use the top 100 logprobs for contextual calibration (only top 5 are available now). While not directly comparable, we report their results for context.}\label{tab:icl}
\vspace{-5pt}
\end{table*}

\paragraph{Results.} Results are presented in Table \ref{tab:icl}. Our gains in quantity while maintaining quality are realized in terms of 64-shot in-context learning accuracy. In some cases, we can generate more examples, but we limit ourselves to 64 for these evaluations for cost and efficiency reasons. Our results at 64 shots are comparable to real data at 64 shots. Notably, our synthetic data at 64 shots improves over real data at 4 shots -- a rough upper bound on the performance of methods limited to generating 4 examples (e.g., \citet{tang2024privacypreserving}). We also improve over results reported in \citet{tang2024privacypreserving} -- however as there are differences in the experimental setup, we also report the results of our re-implementation.\footnote{Specifically, we use their best hyperparameters (from Appendix E, Table 9 of \cite{tang2024privacypreserving}) and algorithm, but with our model, prompt, and evaluation setup.}

We evaluate with GPT-3 babbage-002 which has a 16K context window. We report results on \emph{AGNews}, \emph{DBPedia}, \emph{TREC}, \emph{MIT-G}, and \emph{MIT-D} using the implementation of \citet{zhao2021calibrate}. Following the work of \citet{tang2024privacypreserving}, we enable contextual calibration \citep{zhao2021calibrate} for classification but not extraction tasks. Our evaluation setup is a best-effort reproduction of their setup, which is no longer possible to completely reproduce due to changes to OpenAI API access (see Table \ref{tab:icl} caption). Due to cost, we follow prior work \citep{bertsch2024context,ratner2023parallel,lu2022fantastically,zhao2021calibrate} and opt to subsample test sets to 250 test examples. We run 3 seeds of sampling of exemplars from synthetic/real data. Additionally, we present a limited set of results on Gemma 2 2/9/27B IT, studying the effect of model size on classification performance in \S\ref{appendix:more-experiments:size}.

\subsection{Fine-tuning}\label{sec:experiments:fine-tuning}

We achieve significant improvements over the best available private inference method for in-context learning tasks. Since our method is capable of generating thousands of synthetic examples at reasonable privacy budgets, it is natural to ask whether it can compete with state-of-the-art private fine-tuning methods, which can generate infinitely many synthetic examples once the up-front costs of model training are paid. This makes them capable of producing enough data to train downstream classification models.

\paragraph{Experiment setup.} We use our approach to generate a large quantity of synthetic data for the purposes of fine-tuning 110M BERT-Base models. We consider 3 classification tasks used in prior work on private fine-tuning \citep{kurakin2024harnessing}), following the exact same evaluation procedure. We omit comparison to prior private prediction work (e.g. \citep{tang2024privacypreserving}), as they only generate 4 examples which is insufficient for fine-tuning.

\paragraph{Results.} Main results are presented in Table \ref{tab:fine-tuning}. Across various datasets and privacy levels, we generate between 2.5K (IMDB, $\varepsilon=1$) and 200K (Yelp, $\varepsilon=10$) examples for fine-tuning.  Prior work generating fewer than $10$ examples using private prediction were unable to compare with private fine-tuning on these tasks. While there remains a gap between the best fine-tuning and best private inference methods on downstream classification tasks, we achieve reasonable performance, even out-performing or matching the baseline of privately tuning all the parameters in the model reported in \citet{kurakin2024harnessing}. 

\begin{table*}[t]
    \centering
    \scalebox{0.66}{
    \centering
    \tabcolsep=0.1cm
    \begin{tabular}{lllrrrrrrrrrrrr}
    \toprule
    & & &\multicolumn{12}{c}{BERT Acc. (\%)} \\
    \cmidrule(r){4-15}
    & & &\multicolumn{4}{c}{IMDB @ $\varepsilon$} & \multicolumn{4}{c}{Yelp @ $\varepsilon$} & \multicolumn{4}{c}{AGNews @ $\varepsilon$} \\
    \cmidrule(r){4-7}  \cmidrule(r){8-11}  \cmidrule(r){12-15}
    Method & Reported in & Model &\multicolumn{1}{c}{$\infty$} &\multicolumn{1}{c}{$1$} & \multicolumn{1}{c}{$3$} &\multicolumn{1}{c}{$10$} &\multicolumn{1}{c}{$\infty$} &\multicolumn{1}{c}{$1$} & \multicolumn{1}{c}{$3$} &\multicolumn{1}{c}{$10$} &\multicolumn{1}{c}{$\infty$} &\multicolumn{1}{c}{$1$} & \multicolumn{1}{c}{$3$} &\multicolumn{1}{c}{$10$} \\
    \midrule
    Real data & \citep{kurakin2024harnessing} & - & $93.7_{ 0.1}$ &  \multicolumn{1}{c}{-} & \multicolumn{1}{c}{-} & \multicolumn{1}{c}{-} & $97.6_{ 0.1}$ & \multicolumn{1}{c}{-} & \multicolumn{1}{c}{-} & \multicolumn{1}{c}{-} & $93.7_{ 0.1}$ & \multicolumn{1}{c}{-} & \multicolumn{1}{c}{-} & \multicolumn{1}{c}{-}\\
    \midrule
    Fine-tune & \multirow{3}{*}{\citep{kurakin2024harnessing}} &  \multirow{3}{*}{LaMDA 8B} &  $93.2_{0.2}$ & $79.1_{1.7}$ & $83.9_{0.6}$ & $84.0_{0.7}$ & $95.9_{0.1}$ & $84.1_{0.3}$ & $84.6_{0.1}$  & $84.2_{0.3}$ &  $91.1_{0.1}$ & $65.7_{2.9}$ & $65.3_{2.7}$ & $65.1_{5.3}$\\
    Prompt-tune &  &  & $92.0_{0.1}$ & $88.1_{0.4}$ & $87.4_{0.2}$ & $90.7_{0.2}$ & $93.9_{0.1}$  & $94.1_{0.1}$ & $93.5_{0.1}$  & $94.1_{0.1}$ & $88.3_{0.3}$ & $83.9_{0.8}$ & $86.2_{0.2}$ & $86.9_{0.1}$ \\
    LoRA &  & & $91.6_{0.2}$ & $90.0_{0.3}$ & $90.6_{0.2}$ & $91.3_{0.2}$ & $96.4_{0.1}$ & $95.5_{0.1}$ & $95.6_{0.1}$ & $95.9_{0.1}$ & $91.8_{0.2}$ & $89.4_{0.1}$ & $89.6_{0.1}$ & $90.0_{0.1}$\\
    \midrule
    Ours & This work & Gemma 1.1 2B IT & $83.6_{2.9}$ & $82.7_{2.1}$ & $83.6_{1.9}$  & $85.5_{2.3}$& $91.8_{0.6}$ & $91.1_{0.2}$ & $91.6_{0.8}$ & $92.6_{0.2}$ & $81.2_{1.2}$ & $79.8_{1.8}$ & $79.3_{2.1}$ & $79.8_{0.3}$  \\
    + SVT & This work & Gemma 1.1 2B IT & \multicolumn{1}{c}{-}  & $84.3_{1.1}$ & $84.4_{1.5}$  & $85.0_{1.0}$ & \multicolumn{1}{c}{-}  & $88.4_{0.6}$ & $89.1_{0.3}$ & $89.0_{1.9}$& \multicolumn{1}{c}{-} & $79.2_{0.3}$ & $79.8_{0.4}$ & $80.4_{0.6}$ \\
    \bottomrule
    \end{tabular}
     }
    \caption{\small Results of fine-tuning on real and synthetic data with BERT. We report mean and standard deviation over 3 runs of downstream fine-tuning and evaluation. We compare to results reported in \citep{kurakin2024harnessing} that fine-tunes a synthetic data generator with DP-SGD. We generate 2.5-200K examples with private prediction, which suffices to train reasonably performing models on.}\label{tab:fine-tuning}
\end{table*}

\paragraph{Limited data regime.} We additionally consider the limited data regime. In \S\ref{appendix:more-experiments:limited} we present experiments on \emph{AGNews1K}, a 1024-subsample of \emph{AGNews}. Our method, which employs parallel composition, is ``pay-as-you-go'', i.e., we can put in a small amount of data to get out a small amount, while preserving quality. On the other hand, fine-tuning based approaches necessarily pay upfront to ensure the model and all future generations are private. This means that without sufficient data, all outputs will be low quality. Results in Table \ref{tab:agnews1k} demonstrate that our private prediction method generates more useful examples for in-context learning in this limited data regime.

\subsection{Structured data}\label{sec:experiments:structured}

We conclude our experiments with a demonstration of the lift in performance provided by using the sparse vector technique (SVT) against a public prompt. Informally, the privacy loss of our method only scales with the information density of a new example \emph{vis-a-vis} the public prompt. This contrasts with other private inference methods that incur privacy loss on every token. This is especially useful for structured data, where we avoid incurring privacy loss on syntactic elements of the data. 

\paragraph{Experiment setup.} For JSON generation, we evaluate on a dataset of information about American movies scraped from Wikipedia \citep{rust2024wikipediamoviedata}. Entries contain fields such as \texttt{title}, \texttt{year}, \texttt{cast}, and \texttt{extract} (a short synopsis). We lightly curate the data: we omit uninteresting fields (i.e., thumbnail dimensions and URLs) and remove entries with incomplete entries. We refer to the resulting 34,266 JSON examples with 6 fields as \emph{WikiMoviesJSON}. We evaluate two criteria: the rate at which output generated constitutes well-formed JSON (\emph{Parses (\%)}), and rate at which the output passes basic schema validation (\emph{Validates (\%)}). This includes checks such as: no extra fields, all required fields are present, values are the correct type, and other custom constraints (e.g. no whitespace in the \texttt{href} field).
\paragraph{Results.} Results are in Table \ref{tab:json}. Targeting a large number of examples at small $\varepsilon$ necessitates increases in the sampling temperature $\tau$, to ensure privacy, but compromises the well-formed-ness of outputs. For structured generation, there is a large amount of tokens that (a) are crucial to get right for structure preservation, and (b) easily predictable without access to sensitive data. Here the SVT enables us to get these tokens reliably and for free, leading to better generation quantity.

\begin{table}[!t]
    \centering
    \scalebox{0.7}{
    \centering
    \begin{tabular}{cllrrr}
    \toprule
    \multicolumn{1}{c}{$\varepsilon$} & Method & $\tau$ & \multicolumn{1}{c}{Parses (\%)} & \multicolumn{1}{c}{Validates (\%)} & $m$ \\
    \midrule
    \multirow{6}{*}{1} & \multirow{2}{*}{Ours} & 2 & $80.6_{1.3}$ & $74.2_{1.9}$ & $94.3_{1.2}$  \\
    &  & 2.5  & $4.9_{1.1}$ & $1.5_{0.1}$ & $138.0_{7.5}$ \\
    \cmidrule(r){2-6}
    & \multirow{2}{*}{+ SVT, $\theta$ = 0.9} & 2 & $91.7_{2.1}$ & $88.6_{3.2}$ & $289.7_{19.4}$ \\
    &  & 2.5 & $74.1_{2.7}$ & $64.0_{4.1}$ & $356.7_{25.9}$ \\
    \cmidrule(r){2-6}
    &  \multirow{2}{*}{+ SVT, $\theta$ = 1.5} & 2 & $95.5_{1.0}$ & $93.1_{0.7}$ & $893.0_{20.2}$  \\
    &  & 2.5 & $79.3_{1.0}$ & $72.7_{1.4}$ & $1178.3_{10.1}$ \\
    \bottomrule
    \end{tabular}
    }
    \caption{\small Results for generating JSON records from  \emph{WikiMoviesJSON}. We report mean and standard deviation over 3 runs of dataset generation. $\tau$ refers to the sampling temperature, and $m$ refers to the number of raw samples produced (before parsing and validation checks). The batch size used is 255. We present results at two different SVT thresholds $\theta$, and see gains in structure preservation and quantity.}
    \label{tab:json}
\end{table}

\section{Discussion}
\label{sec:discussion}

We believe that our significantly improved performance relative to \citet{tang2024privacypreserving} is primarily attributable to two algorithmic innovations.

First, for each generated token, \citet{tang2024privacypreserving}  preserve the privacy of the entire distributions from which the token is sampled (by taking argmax), even though only the token itself is included in the synthetic data. By contrast, our method uses a discrete choosing mechanism, the exponential mechanism. As a result, we do not need to maintain a DP version of the entire token distribution to release a single token. This decision leads to significantly lower noise requirements, as a straightforward calculation reveals.
Empirically, we obtained good synthetic data quality with $s = 250$, $\tau = 2$, $c = 10$ and $\delta = 10^{-6}$. In order to switch to the Gaussian mechanism using its standard $(\varepsilon, \delta)$-DP guarantee, and achieve comparable privacy guarantees we would would require $\sigma \approx 0.53$ to achieve a comparable privacy guarantee. (See \S\ref{appendix:noise}). Better analyses of the Gaussian mechanism exist, but do not offer much help. Using the improved analysis in \citet{balle2018improving} to attain the same $\varepsilon$ would require $\sigma \approx 0.34$. Conducting the analysis so that both mechanisms have equivalent privacy loss under zCDP yields $\sigma = 0.2$. These are all very large noise magnitudes relative to probabilities in $[0,1]$.\footnote{To put independent noise of magnitude $\sigma=0.2$ into perspective: suppose the ground truth next-token prediction is deterministic, i.e., $\bbap = [1, 0,...,0] \in \mathbb R^v$, $v$ = 256128 in the case of Gemma. Now with probability $\geq 0.15$, the noised distribution ${\widetilde \bp}$ has ${\widetilde \bp}_1 < 0.8$. Each other $\bp_i$ is $\geq 0.8$ w.p. $\geq 3\cdot 10^{-5}$ independently. Hence the probability of one of these being promoted to argmax is $\geq 0.15 \cdot (1 - (1-3 \cdot 10^{-5})^{v-1}) \approx 0.15$. At this rate, the chance of generating a 30 token span without a corruption is  $<1\%$.}

Secondly, \citet{tang2024privacypreserving} generated each token using a different random sample of the sensitive prompts, which is computationally very expensive, as it prevents the use of KV cache-accelerated decoding, since the cache is invalidated upon every resample. While resampling less often would be more practical, \citet{tang2024privacypreserving} noted that in this case the privacy amplification benefits of subsampling would not be adequately realized, and characterized this limitation as the ``main weakness'' of their approach. Instead, our method generates each synthetic example using a fixed disjoint subset of the sensitive prompts, allowing us to leverage parallel composition in our analysis, and thus avoid this privacy versus computation tradeoff.

\section{Conclusion}

As proprietary models become increasingly powerful, we anticipate more practitioners will be able to generate  inferences from state-of-the-art models, while fewer practitioners will be able to \emph{train} networks that perform like state-of-the-art models. This makes it increasingly important to develop private prediction methods that compete with private fine-tuning. 

We demonstrate that private prediction can be used to generate large amounts of synthetic text with reasonable differential privacy guarantees. We produce 2-3 orders of magnitude more private synthetic data than what was demonstrated in prior work in this paradigm. Access to more synthetic data lets us fine-tune downstream models, as well as yields performance improvements via many-shot in-context learning. Furthermore, we introduce a novel use of public models in which we are able to sample predictable tokens at no privacy cost, which is particularly effective for structured data.

\section*{Limitations}

While our work demonstrates that private prediction is a practical technique for privately generating a large volume of high-quality synthetic data, there remains a gap between our results and the results obtained from privately fine-tuning the parameters of the LLM. Currently, private prediction methods pay a privacy cost for every generated token, while private fine-tuning methods do not. Finally, any method for ensuring data privacy will inevitably entail some loss of data utility.

\section*{Author contributions}\label{sec:contributions}

\begin{itemize}
    \item \textbf{Alex B} is the main contributor. He implemented the method, tested variants to optimize utility and privacy, and ran most of the experiments. He also proposed the use of sparse vector.\vspace{-4pt}

    \item \textbf{Umar} proposed the method, the use of sampling to preserve privacy, and conducted the theoretical analysis.\vspace{-4pt}

    \item \textbf{Umar} and \textbf{Kareem} framed the structure of the paper and led writing.    \vspace{-4pt}

    \item \textbf{Kareem} proposed parallel composition. He also assisted with the privacy analysis.\vspace{-4pt}

    \item \textbf{Natalia} proposed logits recentering.\vspace{-4pt}

    \item \textbf{Weiwei} and \textbf{Alexey} provided infrastructure support and code for running experiments. \textbf{Alexey} suggested the limited data experiments and ran the fine-tuning baselines.\vspace{-4pt}

    \item \textbf{Natalia}, \textbf{Andreas}, and \textbf{Sergei} advised the project.\vspace{-4pt}

    \item \textbf{Everyone} contributed to discussing, interpreting, and iterating on experiment results as well as project management.\vspace{-4pt}
\end{itemize}

\newpage
\bibliography{references}

\begin{thebibliography}{39}
\providecommand{\natexlab}[1]{#1}

\bibitem[{Abadi et~al.(2016)Abadi, Chu, Goodfellow, McMahan, Mironov, Talwar, and Zhang}]{abadi2016deep}
Martin Abadi, Andy Chu, Ian Goodfellow, H~Brendan McMahan, Ilya Mironov, Kunal Talwar, and Li~Zhang. 2016.
\newblock Deep learning with differential privacy.
\newblock In \emph{Proceedings of the 2016 ACM SIGSAC conference on computer and communications security}, pages 308--318.

\bibitem[{Balle and Wang(2018)}]{balle2018improving}
Borja Balle and Yu-Xiang Wang. 2018.
\newblock Improving the gaussian mechanism for differential privacy: Analytical calibration and optimal denoising.
\newblock In \emph{International Conference on Machine Learning}, pages 394--403. PMLR.

\bibitem[{Bertsch et~al.(2024)Bertsch, Ivgi, Alon, Berant, Gormley, and Neubig}]{bertsch2024context}
Amanda Bertsch, Maor Ivgi, Uri Alon, Jonathan Berant, Matthew~R Gormley, and Graham Neubig. 2024.
\newblock In-context learning with long-context models: An in-depth exploration.
\newblock \emph{arXiv preprint arXiv:2405.00200}.

\bibitem[{Bun and Steinke(2016)}]{bun2016concentrated}
Mark Bun and Thomas Steinke. 2016.
\newblock Concentrated differential privacy: Simplifications, extensions, and lower bounds.
\newblock In \emph{Theory of Cryptography Conference}, pages 635--658. Springer.

\bibitem[{Canonne et~al.(2020)Canonne, Kamath, and Steinke}]{canonne2020discrete}
Cl{\'e}ment~L Canonne, Gautam Kamath, and Thomas Steinke. 2020.
\newblock The discrete gaussian for differential privacy.
\newblock \emph{Advances in Neural Information Processing Systems}, 33:15676--15688.

\bibitem[{Cesar and Rogers(2021)}]{cesar2021bounding}
Mark Cesar and Ryan Rogers. 2021.
\newblock \href {https://proceedings.mlr.press/v132/cesar21a.html} {Bounding, concentrating, and truncating: Unifying privacy loss composition for data analytics}.
\newblock In \emph{Proceedings of the 32nd International Conference on Algorithmic Learning Theory}, volume 132 of \emph{Proceedings of Machine Learning Research}, pages 421--457. PMLR.

\bibitem[{Duan et~al.(2023)Duan, Dziedzic, Papernot, and Boenisch}]{NEURIPS2023_f26119b4}
Haonan Duan, Adam Dziedzic, Nicolas Papernot, and Franziska Boenisch. 2023.
\newblock \href {https://proceedings.neurips.cc/paper_files/paper/2023/file/f26119b4ffe38c24d97e4c49d334b99e-Paper-Conference.pdf} {Flocks of stochastic parrots: Differentially private prompt learning for large language models}.
\newblock In \emph{Advances in Neural Information Processing Systems}, volume~36, pages 76852--76871. Curran Associates, Inc.

\bibitem[{Durfee and Rogers(2019)}]{durfee2019practical}
David Durfee and Ryan~M Rogers. 2019.
\newblock Practical differentially private top-k selection with pay-what-you-get composition.
\newblock \emph{Advances in Neural Information Processing Systems}, 32.

\bibitem[{Dwork and Feldman(2018)}]{pmlr-v75-dwork18a}
Cynthia Dwork and Vitaly Feldman. 2018.
\newblock \href {https://proceedings.mlr.press/v75/dwork18a.html} {Privacy-preserving prediction}.
\newblock In \emph{Proceedings of the 31st Conference On Learning Theory}, volume~75 of \emph{Proceedings of Machine Learning Research}, pages 1693--1702. PMLR.

\bibitem[{Dwork et~al.(2006)Dwork, Kenthapadi, McSherry, Mironov, and Naor}]{dwork2006our}
Cynthia Dwork, Krishnaram Kenthapadi, Frank McSherry, Ilya Mironov, and Moni Naor. 2006.
\newblock Our data, ourselves: Privacy via distributed noise generation.
\newblock In \emph{Advances in Cryptology-EUROCRYPT 2006: 24th Annual International Conference on the Theory and Applications of Cryptographic Techniques, St. Petersburg, Russia, May 28-June 1, 2006. Proceedings 25}, pages 486--503. Springer.

\bibitem[{Dwork et~al.(2009)Dwork, Naor, Reingold, Rothblum, and Vadhan}]{dwork2009complexity}
Cynthia Dwork, Moni Naor, Omer Reingold, Guy~N Rothblum, and Salil Vadhan. 2009.
\newblock On the complexity of differentially private data release: efficient algorithms and hardness results.
\newblock In \emph{Proceedings of the forty-first annual ACM symposium on Theory of computing}, pages 381--390.

\bibitem[{Flemings et~al.(2024)Flemings, Razaviyayn, and Annavaram}]{flemings2024differentially}
James Flemings, Meisam Razaviyayn, and Murali Annavaram. 2024.
\newblock \href {https://arxiv.org/abs/2403.15638} {Differentially private next-token prediction of large language models}.
\newblock \emph{Preprint}, arXiv:2403.15638.

\bibitem[{{Gemma Team}(2024)}]{gemmateam2024gemma}
{Gemma Team}. 2024.
\newblock \href {https://arxiv.org/abs/2403.08295} {Gemma: Open models based on gemini research and technology}.
\newblock \emph{Preprint}, arXiv:2403.08295.

\bibitem[{Ginart et~al.(2022)Ginart, van~der Maaten, Zou, and Guo}]{ginart2022submix}
Antonio Ginart, Laurens van~der Maaten, James Zou, and Chuan Guo. 2022.
\newblock \href {https://arxiv.org/abs/2201.00971} {Submix: Practical private prediction for large-scale language models}.
\newblock \emph{CoRR}, abs/2201.00971.

\bibitem[{Hong et~al.(2024)Hong, Wang, Zhang, LI, Li, and Wang}]{hong2024dpopt}
Junyuan Hong, Jiachen~T. Wang, Chenhui Zhang, Zhangheng LI, Bo~Li, and Zhangyang Wang. 2024.
\newblock \href {https://openreview.net/forum?id=Ifz3IgsEPX} {{DP}-{OPT}: Make large language model your privacy-preserving prompt engineer}.
\newblock In \emph{The Twelfth International Conference on Learning Representations}.

\bibitem[{Hu et~al.(2022)Hu, yelong shen, Wallis, Allen-Zhu, Li, Wang, Wang, and Chen}]{hu2022lora}
Edward~J Hu, yelong shen, Phillip Wallis, Zeyuan Allen-Zhu, Yuanzhi Li, Shean Wang, Lu~Wang, and Weizhu Chen. 2022.
\newblock \href {https://openreview.net/forum?id=nZeVKeeFYf9} {Lo{RA}: Low-rank adaptation of large language models}.
\newblock In \emph{International Conference on Learning Representations}.

\bibitem[{Kurakin et~al.(2024)Kurakin, Ponomareva, Syed, MacDermed, and Terzis}]{kurakin2024harnessing}
Alexey Kurakin, Natalia Ponomareva, Umar Syed, Liam MacDermed, and Andreas Terzis. 2024.
\newblock \href {https://arxiv.org/abs/2306.01684} {Harnessing large-language models to generate private synthetic text}.
\newblock \emph{Preprint}, arXiv:2306.01684.

\bibitem[{Liu et~al.(2012)Liu, Cyphers, Pasupat, McGraw, and Glass}]{liu2012conversational}
Jingjing Liu, Scott Cyphers, Panupong Pasupat, Ian McGraw, and James~R. Glass. 2012.
\newblock \href {https://doi.org/10.21437/INTERSPEECH.2012-563} {A conversational movie search system based on conditional random fields}.
\newblock In \emph{{INTERSPEECH} 2012, 13th Annual Conference of the International Speech Communication Association, Portland, Oregon, USA, September 9-13, 2012}, pages 2454--2457. {ISCA}.

\bibitem[{Lu et~al.(2022)Lu, Bartolo, Moore, Riedel, and Stenetorp}]{lu2022fantastically}
Yao Lu, Max Bartolo, Alastair Moore, Sebastian Riedel, and Pontus Stenetorp. 2022.
\newblock \href {https://doi.org/10.18653/v1/2022.acl-long.556} {Fantastically ordered prompts and where to find them: Overcoming few-shot prompt order sensitivity}.
\newblock In \emph{Proceedings of the 60th Annual Meeting of the Association for Computational Linguistics (Volume 1: Long Papers)}, pages 8086--8098, Dublin, Ireland. Association for Computational Linguistics.

\bibitem[{Maas et~al.(2011)Maas, Daly, Pham, Huang, Ng, and Potts}]{maas2011learning}
Andrew~L. Maas, Raymond~E. Daly, Peter~T. Pham, Dan Huang, Andrew~Y. Ng, and Christopher Potts. 2011.
\newblock \href {https://aclanthology.org/P11-1015} {Learning word vectors for sentiment analysis}.
\newblock In \emph{Proceedings of the 49th Annual Meeting of the Association for Computational Linguistics: Human Language Technologies}, pages 142--150, Portland, Oregon, USA. Association for Computational Linguistics.

\bibitem[{Majmudar et~al.(2022)Majmudar, Dupuy, Peris, Smaili, Gupta, and Zemel}]{Majmudar2022}
Jimit Majmudar, Christophe Dupuy, Charith Peris, Sami Smaili, Rahul Gupta, and Richard Zemel. 2022.
\newblock \href {https://www.amazon.science/publications/differentially-private-decoding-in-large-language-models} {Differentially private decoding in large language models}.
\newblock In \emph{NAACL 2022 Second Workshop on Trustworthy Natural Language Processing (TrustNLP)}.

\bibitem[{Nissim et~al.(2007)Nissim, Raskhodnikova, and Smith}]{nrs07}
Kobbi Nissim, Sofya Raskhodnikova, and Adam~D. Smith. 2007.
\newblock Smooth sensitivity and sampling in private data analysis.
\newblock In \emph{Proceedings of the 39th Annual {ACM} Symposium on Theory of Computing, San Diego, California, USA, June 11-13, 2007}, pages 75--84. {ACM}.

\bibitem[{Papernot et~al.(2017)Papernot, Abadi, Erlingsson, Goodfellow, and Talwar}]{papernot2017semi}
Nicolas Papernot, Mart{\'{\i}}n Abadi, {\'{U}}lfar Erlingsson, Ian~J. Goodfellow, and Kunal Talwar. 2017.
\newblock \href {https://openreview.net/forum?id=HkwoSDPgg} {Semi-supervised knowledge transfer for deep learning from private training data}.
\newblock In \emph{5th International Conference on Learning Representations, {ICLR} 2017, Toulon, France, April 24-26, 2017, Conference Track Proceedings}.

\bibitem[{Papernot et~al.(2018)Papernot, Song, Mironov, Raghunathan, Talwar, and Erlingsson}]{papernot2018scalable}
Nicolas Papernot, Shuang Song, Ilya Mironov, Ananth Raghunathan, Kunal Talwar, and {\'{U}}lfar Erlingsson. 2018.
\newblock \href {https://openreview.net/forum?id=rkZB1XbRZ} {Scalable private learning with {PATE}}.
\newblock In \emph{6th International Conference on Learning Representations, {ICLR} 2018, Vancouver, BC, Canada, April 30 - May 3, 2018, Conference Track Proceedings}.

\bibitem[{Ponomareva et~al.(2023)Ponomareva, Hazimeh, Kurakin, Xu, Denison, McMahan, Vassilvitskii, Chien, and Thakurta}]{ponomareva2023how}
Natalia Ponomareva, Hussein Hazimeh, Alex Kurakin, Zheng Xu, Carson Denison, H.~Brendan McMahan, Sergei Vassilvitskii, Steve Chien, and Abhradeep~Guha Thakurta. 2023.
\newblock \href {https://doi.org/10.1613/jair.1.14649} {How to dp-fy {ML:} {A} practical guide to machine learning with differential privacy}.
\newblock \emph{J. Artif. Intell. Res.}, 77:1113--1201.

\bibitem[{Ratner et~al.(2023)Ratner, Levine, Belinkov, Ram, Magar, Abend, Karpas, Shashua, Leyton-Brown, and Shoham}]{ratner2023parallel}
Nir Ratner, Yoav Levine, Yonatan Belinkov, Ori Ram, Inbal Magar, Omri Abend, Ehud Karpas, Amnon Shashua, Kevin Leyton-Brown, and Yoav Shoham. 2023.
\newblock \href {https://doi.org/10.18653/v1/2023.acl-long.352} {Parallel context windows for large language models}.
\newblock In \emph{Proceedings of the 61st Annual Meeting of the Association for Computational Linguistics (Volume 1: Long Papers)}, pages 6383--6402, Toronto, Canada. Association for Computational Linguistics.

\bibitem[{Rogers and Steinke(2021)}]{DPorg-exponential-mechanism-bounded-range}
Ryan Rogers and Thomas Steinke. 2021.
\newblock A better privacy analysis of the exponential mechanism.
\newblock DifferentialPrivacy.org.
\newblock \url{https://differentialprivacy.org/exponential-mechanism-bounded-range/}.

\bibitem[{Rust(2024)}]{rust2024wikipediamoviedata}
Peter Rust. 2024.
\newblock wikipedia-movie-data.
\newblock \url{https://github.com/prust/wikipedia-movie-data}.

\bibitem[{Tang et~al.(2024)Tang, Shin, Inan, Manoel, Mireshghallah, Lin, Gopi, Kulkarni, and Sim}]{tang2024privacypreserving}
Xinyu Tang, Richard Shin, Huseyin~A Inan, Andre Manoel, Fatemehsadat Mireshghallah, Zinan Lin, Sivakanth Gopi, Janardhan Kulkarni, and Robert Sim. 2024.
\newblock \href {https://openreview.net/forum?id=oZtt0pRnOl} {Privacy-preserving in-context learning with differentially private few-shot generation}.
\newblock In \emph{The Twelfth International Conference on Learning Representations}.

\bibitem[{Tran and Xiong(2024)}]{tran2024differentially}
Toan~V. Tran and Li~Xiong. 2024.
\newblock \href {https://arxiv.org/abs/2406.01457} {Differentially private tabular data synthesis using large language models}.
\newblock \emph{Preprint}, arXiv:2406.01457.

\bibitem[{Turc et~al.(2019)Turc, Chang, Lee, and Toutanova}]{turc2019}
Iulia Turc, Ming-Wei Chang, Kenton Lee, and Kristina Toutanova. 2019.
\newblock Well-read students learn better: On the importance of pre-training compact models.
\newblock \emph{arXiv preprint arXiv:1908.08962v2}.

\bibitem[{Voorhees and Tice(2000)}]{voorhees2000building}
Ellen~M. Voorhees and Dawn~M. Tice. 2000.
\newblock \href {https://doi.org/10.1145/345508.345577} {Building a question answering test collection}.
\newblock In \emph{Proceedings of the 23rd Annual International ACM SIGIR Conference on Research and Development in Information Retrieval}, SIGIR '00, page 200–207, New York, NY, USA. Association for Computing Machinery.

\bibitem[{Wu et~al.(2024{\natexlab{a}})Wu, Xu, Zhang, Zhang, and Ramage}]{wu2024prompt}
Shanshan Wu, Zheng Xu, Yanxiang Zhang, Yuanbo Zhang, and Daniel Ramage. 2024{\natexlab{a}}.
\newblock \href {https://arxiv.org/abs/2404.04360} {Prompt public large language models to synthesize data for private on-device applications}.
\newblock \emph{Preprint}, arXiv:2404.04360.

\bibitem[{Wu et~al.(2024{\natexlab{b}})Wu, Panda, Wang, and Mittal}]{wu2024privacypreserving}
Tong Wu, Ashwinee Panda, Jiachen~T. Wang, and Prateek Mittal. 2024{\natexlab{b}}.
\newblock \href {https://openreview.net/forum?id=x4OPJ7lHVU} {Privacy-preserving in-context learning for large language models}.
\newblock In \emph{The Twelfth International Conference on Learning Representations}.

\bibitem[{Xie et~al.(2024)Xie, Lin, Backurs, Gopi, Yu, Inan, Nori, Jiang, Zhang, Lee, Li, and Yekhanin}]{xie2024differentially}
Chulin Xie, Zinan Lin, Arturs Backurs, Sivakanth Gopi, Da~Yu, Huseyin~A Inan, Harsha Nori, Haotian Jiang, Huishuai Zhang, Yin~Tat Lee, Bo~Li, and Sergey Yekhanin. 2024.
\newblock \href {https://openreview.net/forum?id=jnF53uXmBS} {Differentially private synthetic data via foundation model {API}s 2: Text}.
\newblock In \emph{ICLR 2024 Workshop on Secure and Trustworthy Large Language Models}.

\bibitem[{Yu et~al.(2024)Yu, Kairouz, Oh, and Xu}]{yu2024privacypreserving}
Da~Yu, Peter Kairouz, Sewoong Oh, and Zheng Xu. 2024.
\newblock \href {https://arxiv.org/abs/2402.13659} {Privacy-preserving instructions for aligning large language models}.
\newblock \emph{Preprint}, arXiv:2402.13659.

\bibitem[{Yue et~al.(2023)Yue, Inan, Li, Kumar, McAnallen, Shajari, Sun, Levitan, and Sim}]{yue-etal-2023-synthetic}
Xiang Yue, Huseyin Inan, Xuechen Li, Girish Kumar, Julia McAnallen, Hoda Shajari, Huan Sun, David Levitan, and Robert Sim. 2023.
\newblock \href {https://doi.org/10.18653/v1/2023.acl-long.74} {Synthetic text generation with differential privacy: A simple and practical recipe}.
\newblock In \emph{Proceedings of the 61st Annual Meeting of the Association for Computational Linguistics (Volume 1: Long Papers)}, pages 1321--1342, Toronto, Canada. Association for Computational Linguistics.

\bibitem[{Zhang et~al.(2015)Zhang, Zhao, and LeCun}]{zhang2015character}
Xiang Zhang, Junbo Zhao, and Yann LeCun. 2015.
\newblock \href {https://proceedings.neurips.cc/paper_files/paper/2015/file/250cf8b51c773f3f8dc8b4be867a9a02-Paper.pdf} {Character-level convolutional networks for text classification}.
\newblock In \emph{Advances in Neural Information Processing Systems}, volume~28. Curran Associates, Inc.

\bibitem[{Zhao et~al.(2021)Zhao, Wallace, Feng, Klein, and Singh}]{zhao2021calibrate}
Zihao Zhao, Eric Wallace, Shi Feng, Dan Klein, and Sameer Singh. 2021.
\newblock \href {https://proceedings.mlr.press/v139/zhao21c.html} {Calibrate before use: Improving few-shot performance of language models}.
\newblock In \emph{Proceedings of the 38th International Conference on Machine Learning}, volume 139 of \emph{Proceedings of Machine Learning Research}, pages 12697--12706. PMLR.

\end{thebibliography}

\newpage
\appendix

\section{Additional experiments}\label{appendix:more-experiments}
\subsection{Private prediction beats fine-tuning in the limited data regime}\label{appendix:more-experiments:limited}

We do LoRA fine-tuning with DP-SGD on \emph{AGNews1K}, with the \emph{same setup that beats our method in the full data regime}. We sample synthetic data from the fine-tuned model. We also run our private prediction method on \emph{AGNews1K}. We evaluate performance on 4 and 16 shot in-context learning with GPT-3 babbage-002 (the same experimental setting as \S\ref{sec:experiments:icl}). Our private prediction approach outperforms the fine-tuning setup that does better in the full data regime.

\begin{table}[!ht]
    \small
    \centering
    \begin{tabular}{clllr}
    \toprule
    \multicolumn{1}{c}{$\varepsilon$} & Method & Shots & Model & Acc. (\%)  \\
    \midrule
    \multirow{4}{*}{1} & \multirow{2}{*}{LoRA} & 4 & \multirow{2}{*}{LaMDA 8B} & $63.3_{8.0}$\\
     & & 16 &  & $68.1_{5.9}$\\
    \cmidrule(r){2-5}
    & \multirow{2}{*}{Ours} & 4 & \multirow{2}{*}{Gemma 1.1 2B IT} & $73.9_{8.3}$\\
     & & 16 &  & $80.1_{2.5}$\\
    \bottomrule
    \end{tabular}
    \caption{\small Results on \emph{AGNews1K}, a 1024-subsample of AGNews. Our method is ``pay-as-you-go'', and is capable of generating a few high quality examples for in-context learning in this regime. On the other hand fine-tuning does worse due to the stricter requirement that all future model outputs must be private. Evaluation setup is the same in \S\ref{sec:experiments:icl}, except here we run 16 instead of 64-shot, because 16 examples produced by the LoRA model fills up the entire context length of babbage-002.}
    \label{tab:agnews1k}
\end{table}

\subsection{Effect of model size}\label{appendix:more-experiments:size}
We report results on the effect of the data generator's size on in-context classification performance on DBPedia. Our setup is the same as the experiments in \S\ref{sec:experiments:icl}, with the change that we use the Gemma 2 2/9/27B IT models to get more size variation in the same model family. This necessitated a slight change in the prompt (specifically, we append to the instruction: ``{\texttt{No formatting or explanations.}}'' For this evaluation, we see limited improvement due to scale.

\begin{table}[!ht]
    \small
    \centering
    \begin{tabular}{cllr}
    \toprule
    \multicolumn{1}{c}{$\varepsilon$} & Shots & Model & Acc. (\%)  \\
    \midrule
    \multirow{7}{*}{1} & 4 & \multirow{2}{*}{Gemma 2 2B IT} & $75.7_{0.8}$ \\
     & 64 &  & $91.2_{0.0}$\\
    \cmidrule(r){2-4}
     & 4 & \multirow{2}{*}{Gemma 2 9B IT} & $76.4_{1.2}$\\
     &  64 &  & $92.4_{1.4}$\\
    \cmidrule(r){2-4}
     & 4 & \multirow{2}{*}{Gemma 2 27B IT} & $76.9_{2.2}$\\
     &  64 &  & $91.9_{1.9}$\\
    \bottomrule
    \end{tabular}
    \caption{\small Results on DBPedia classification. Evaluation setup is the same as in \S\ref{sec:experiments:icl}. We see limited improvement from the increase in model size.}
    \label{tab:size}
\end{table}

\section{Design choices}\label{appendix:design-choices}

\subsection{Logits clipping function}\label{appendix:design-choices:clipping}

In Figure \ref{fig:recenter}, we compare results for different logits clipping functions. The baseline approach it to clip all logits to the interval $[-c,c]$ before aggregation and softmax -- we refer to this as ``fixed interval clipping''. Alternatively, we can clip to the range $[\max_j \{\bz_j\}-2c, \max_j \{\bz_j\}]$ and then translate to the interval $[-c, c]$ (Eq.~\ref{eq:clip}). In Figure \ref{fig:recenter} we plot the distortion as a consequence of clipping in terms of L1 error, and find that the latter approach allows us clip more than twice as aggressively, thus improving the privacy guarantee, without compromising utility.

\section{Privacy checklist}
We follow privacy reporting guidelines \cite[Section 5.3.3]{ponomareva2023how} to state all information needed to get a complete picture of our privacy guarantee.

\begin{enumerate}
    \item \textbf{DP setting.} This a central DP guarantee where the service provider is trusted
to correctly implement the mechanism.
    \item \textbf{Instantiating the DP definition.}
    \begin{enumerate}
        \item \emph{Data accesses covered.} DP guarantees only apply to a single synthetic dataset generation run. We do not account for hyperparameter tuning.
        \item \emph{Final mechanism output.} The mechanism's output is the full sequence of sampled tokens of the output synthetic dataset (all synthetic data instances), which can viewed as responses to a sequence of adaptive queries with different input text prefixes.
        \item \emph{Unit of privacy.} Example-level DP, where an example is the full sequence of tokens for one input sensitive example.
        \item \emph{Adjacency definition for neighboring datasets.} Add/remove one example.
    \end{enumerate}
    \item \textbf{Privacy accounting details.}
    \begin{enumerate}
        \item \emph{Type of accounting used.} zCDP based accounting.
        \item \emph{Accounting assumptions.} Our analysis assumes independent assignment of each prompt to a batch (Assumption 1). We can fulfill this with hashing (see \S4), which corresponds to assigning each prompt to a batch uniformly at random. However, the implementation uses fixed size batches constructed via shuffling.
        \item \emph{The formal DP statement.} We report $(\varepsilon,\delta)$-DP for $\varepsilon=1,3,10$ and $\delta = (\texttt{training\_set\_size})^{-1}$.
    \end{enumerate}
    \item\textbf{Transparency and verifiability.} We provide a full description of the algorithm and a complete privacy analysis. Our implementation is built off of the official open-source JAX library for Gemma: \url{https://github.com/google-deepmind/gemma}. At this time, the implementation is close-sourced.
\end{enumerate}

\paragraph{Discussion.} The full output synthetic dataset of our mechanism indeed satisfies example-level add/remove DP. Below we discuss the more granular aspects of our privacy guarantee that are not captured in the above statement. 

First, as a consequence of parallel composition, we note that each individual synthetic example output only incurs privacy cost on the batch used to generate it. Hence, having the mechanism release a subset of the output will only incur privacy cost on a subset of our input dataset. Second, our current privacy guarantee holds with respect to the change of an entire context, which may consist of multiple examples. However, as we do not prescribe a way to group examples into contexts, we cannot \emph{a priori} identify which examples belong to a single context.

\begin{figure*}[!t] %
  \centering

  \begin{subfigure}[b]{0.8\linewidth}
    \includegraphics[trim={0 37.5cm 0 0},clip,width=\linewidth]{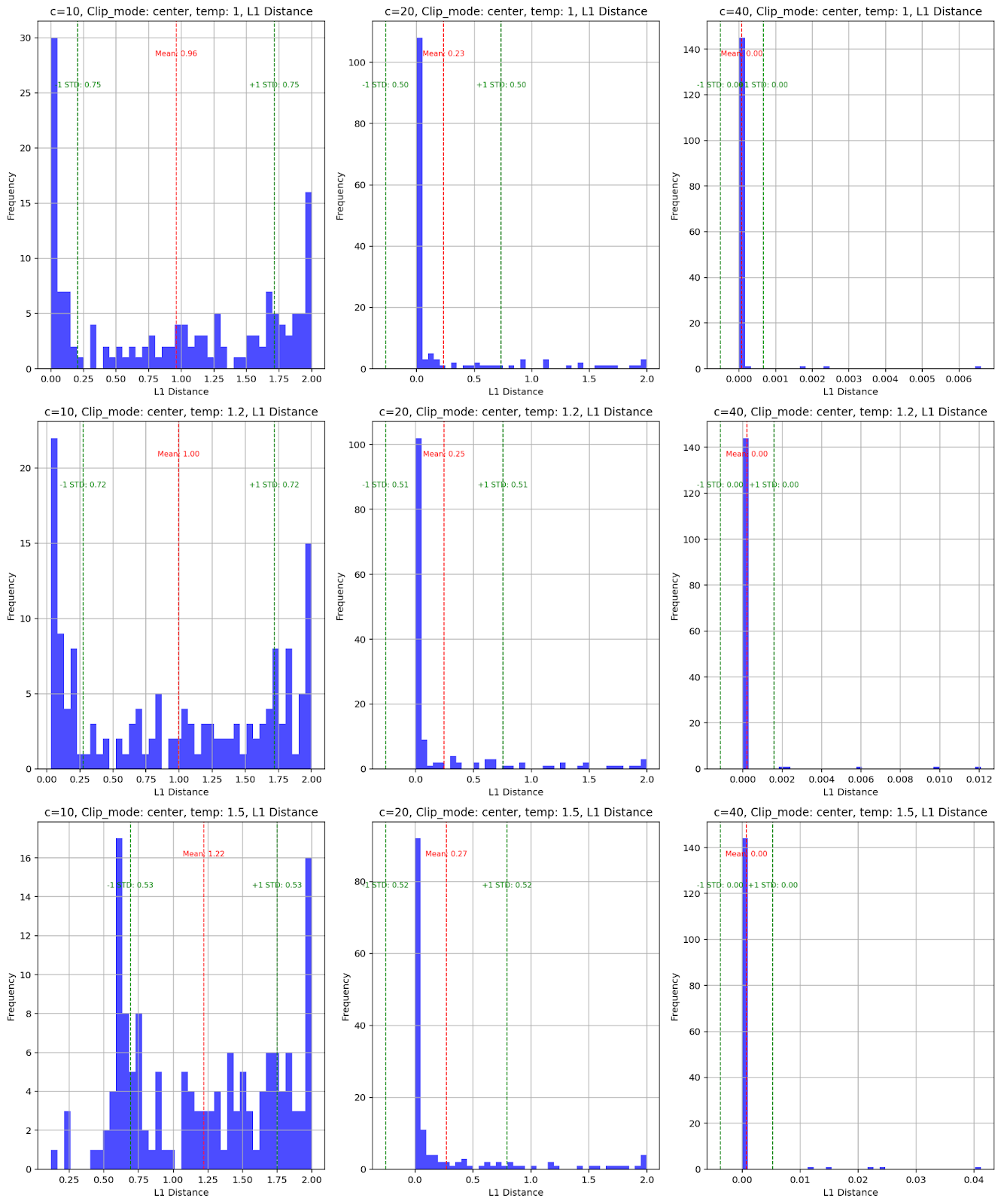}
    \caption{Distribution of L1 error induced by fixed interval clipping.}
  \end{subfigure}
  \vspace{15pt}

  \begin{subfigure}[b]{0.8\linewidth}
    \includegraphics[trim={0 37.5cm 0 0},clip,width=\linewidth]{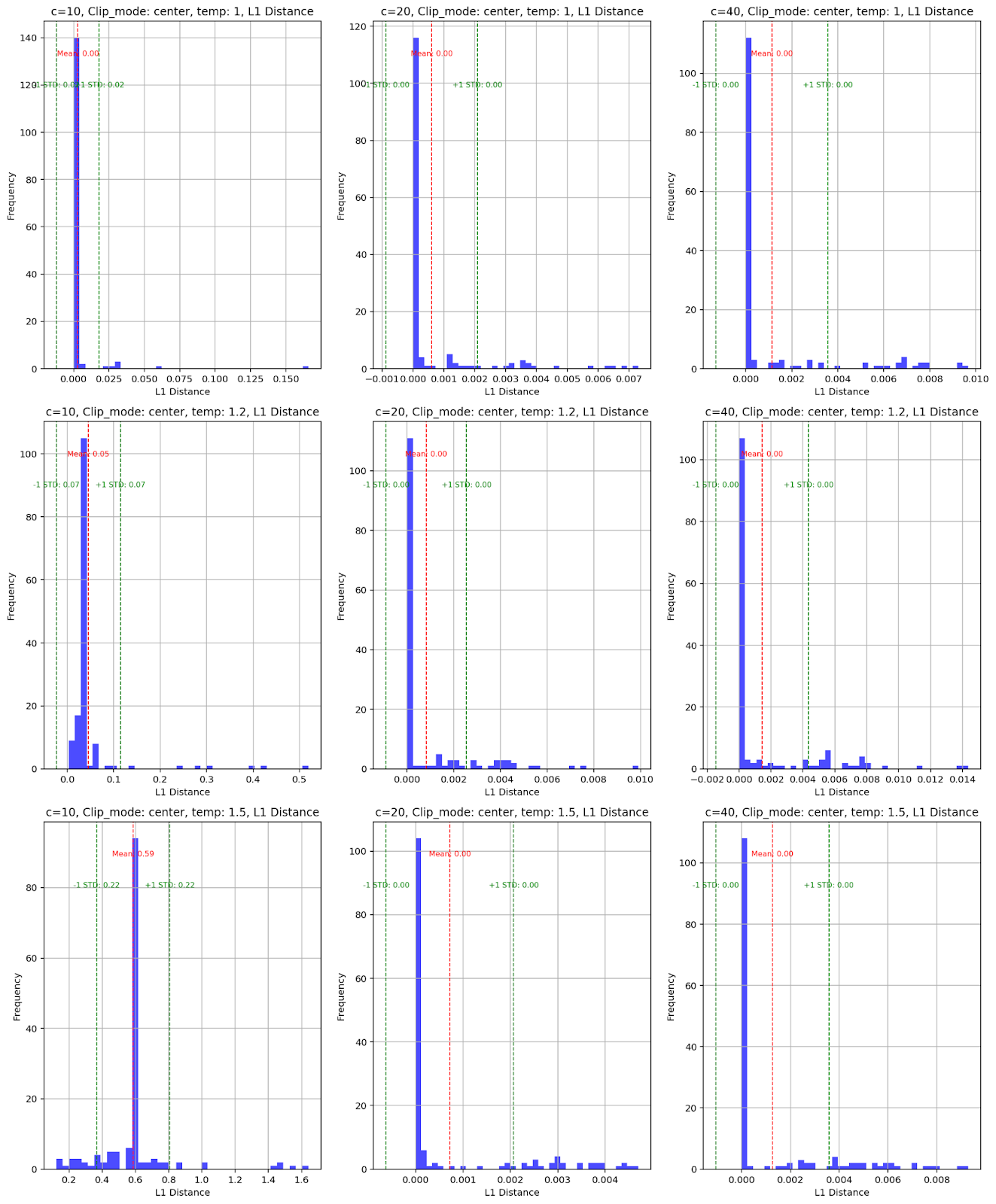}
    \caption{Distribution of L1 error induced by clipping with recentering.}
  \end{subfigure}
  
  \caption{\small We sample a few hundred tokens using logits aggregation with no clipping. At each sampling step, we compute the L1 distances between the post-softmax distributions of aggregated clipped logits vs. aggregated unclipped logits, at various settings of $c$, and plot them in a histrogram. We observe less error, at lower choices of $c$ when clipping with recentering (note the $x$-axis scales).}\label{fig:recenter}
\end{figure*}

\section{Proof of Theorem \ref{thm:main}}\label{appendix:proof}

Our proof of Theorem \ref{thm:main} is organized into sections. \S\ref{sec:definitions} provides basic definitions. \S\ref{sec:composition} and \S\ref{sec:sensitivity} establish key results related to composition and sensitivity. \S\ref{sec:simpler} proves the privacy of simpler mechanisms that each account for a portion of the functionality of Algorithm \ref{alg:main}. \S\ref{sec:together} puts all the pieces together and completes the proof.

\subsection{Definitions}
\label{sec:definitions}

In \S\ref{sec:analysis} we defined neighboring prompt datasets. We extend the definition to arbitrary sets.

\begin{defn} Let $\CU$ be a set. Let $S, S' \subseteq \CU$. We say that $S$ and $S'$ are \emph{neighbors} if there exists $u \in \CU$ such that $S = S' \cup \{u\}$ or $S' = S \cup \{u\}$.\end{defn}

The sensitivity of a function is an upper bound on how much its value can change over neighbors.

\begin{defn} Let $\CU$ be a set. Let $k \ge 1$. Let $f: 2^\CU \rightarrow \bbR^k$. The \emph{sensitivity} of $f$ is
\[
\sup_{S, S'} \norm{f(S) - f(S')}_\infty
\]
where the supremum is over neighbors $S, S' \in \CU$.\end{defn}

Zero-concentrated differential privacy (zCDP) is a relaxation of $\varepsilon$-differential privacy.

\begin{defn}[\citet{bun2016concentrated}] A mechanism $M$ satisfies \emph{$\rho$-zCDP} if
\[
D_\alpha(M(D) \parallel M(D')) \le \rho\alpha
\]
for all $\alpha > 1$ and neighboring datasets $D, D' \in \CD$, where $D_\alpha(P \parallel Q)$ is R\'enyi divergence of order $\alpha$ betweeen distributions $P$ and $Q$.

\end{defn}

\subsection{Composition}
\label{sec:composition}

Zero-concentrated differential privacy obeys a simple sequential composition rule.

\begin{lem} \label{lem:sequential} If mechanisms $M_1$ and $M_2$ satisfy $\rho_1$-zCDP and $\rho_2$-zCDP, respectively, then the sequential composition of $M_1$ and $M_2$ satisfies $(\rho_1 + \rho_2)$-zCDP.\end{lem}

Parallel composition is a well-known technique in differential privacy that is useful for establishing privacy guarantees in scenarios where a mechanism is independently applied to disjoint batches of a dataset. Many versions of parallel composition require that the batches are chosen in a fully data-independent manner. We show that the same result holds under a weaker assumption.

\begin{lem} \label{lem:parallel} Let $k > 1$ be the number of batches. Let $f$ be a batch assignment function that satisfies Assumption \ref{assum:batch}. In other words, let $f$ be a function that map each prompt to and element of $[k]$. For any dataset of prompts $D$ and $i \in [k]$ let
\[
D_i = \{\bp \in D : f(\bp) = i\}.
\]
be the $i$th batch of prompts. Let $M$ be a mechanism that satisfies $\rho$-zCDP. If $\widehat{M}$ is the mechanism defined by
\[
\widehat{M}(D) = (M(D_1), \ldots, M(D_k))
\]
then $\widehat{M}$ satisfies $\rho$-zCDP.
\end{lem}
\begin{proof} Let $D, D' \in \CD$ be neighboring datasets. Without loss of generality assume $D = D' \cup \{\bp\}$, where $\bp$ is a prompt. There exists $j \in [k]$ such that $D_i = D'_i$ for all $i \neq j$ and $D_j = D'_j \cup \{\bp\}$. We have for all $\alpha > 1$
\begin{align*}
D_\alpha(\widehat{M}(D) \| \widehat{M}(D')) &= \sum_{i=1}^k D_\alpha(M(D_i) \| M(D'_i))\\
&= D_\alpha(M(D_j) \| M(D'_j))\\
&\le \rho\alpha \qedhere
\end{align*}
\end{proof}

\subsection{Sensitivity analysis}
\label{sec:sensitivity}

In this we compute the sensitivity of several functions used in Algorithm \ref{alg:main}. Each function depends on a set of logit vectors. Recall that a logit vector is an element of $\bbR^v$. Let
\[
\ell(Z) = \frac1s \sum_{\bz \in Z} \clip_c(\bz)
\]
where $\clip_c(\cdot)$ was defined in Eq.~\eqref{eq:clip}. Also recall the distance function defined in Eq.~\eqref{eq:distance}:
\[
\mathrm{d}(Z, \bz) = \norm{\frac 1 s \sum_{\bz' \in Z} p_{\bz'} - p_{\bz}}_1
\]
where $p_{\bz} = \softmax(\bz)$.

\begin{lem} \label{lem:sensitivity} The function $\ell$ has sensitivity $\frac{c}{s}$, and for all $\bz \in \bbR^v$, the function $\mathrm{d}(\cdot, \bz)$ has sensitivity $\frac1s$.
\end{lem}
\begin{proof} Let $Z, Z' \subseteq \bbR^v$ be neighbors. Let $\tbz \in \bbR^v$ be the logit vector they do not have in common. We have
\[
\norm{\ell(Z) - \ell(Z')}_\infty = \frac1s\norm{\clip_c(\tbz)}_\infty \le \frac{c}{s}.
\]
We also have
\begin{align*}
&\left| \mathrm{d}(Z, \bz) - d(D', \bz) \right|\\
=& \left| \norm{\frac1s \sum_{\bz' \in Z} p_{\bz'} - p_{\bz}}_1 - \norm{\frac1s \sum_{\bz' \in Z'} p_{\bz'} - p_{\bz}}_1\right|\\
\le&\norm{\frac1s \sum_{\bz' \in Z} p_{\bz'} - \frac1s \sum_{\bz' \in Z'} p_{\bz'}}_1\\
=& \norm{\frac1s \bp_{\tbz}}_1\\
=& \frac1s
\end{align*}
where we used the reverse triangle inequality.
\end{proof}

\subsection{Constituent mechanisms}
\label{sec:simpler}

In this section we prove privacy guarantees for several simpler mechanisms that we will later compose together to show that Algorithm \ref{alg:main} is private.

Both Algorithms \ref{alg:privatesampling} and \ref{alg:abovethresholdsampling} accept a sensitive prompt dataset and a token sequence as input. Algorithm \ref{alg:privatesampling} appends a private token to the token sequence, while Algorithm \ref{alg:abovethresholdsampling} appends zero or more public tokens to the token sequence. The operation of both algorithms is governed by the parameters of Algorithm \ref{alg:main} (\emph{e.g.}, temperature, noise level, \emph{etc}).

\begin{algorithm}
\setstretch{1.15}
\caption{\label{alg:privatesampling} Private token generation}
\begin{algorithmic}[1]
\Statex {\bf Input:} Sensitive prompt dataset $D$, initial token sequence $\bx_0$
\Statex {\bf Output:} Token sequence $\bx \in \CX^*$
\State $\bx \gets \bx_0$
\State $Z \gets \{\logits(\bp\bx) : \bp \in D\}$
\State $\bbaz \gets \ell(Z)$
\State $x \sim \softmax(\bbaz / \tau)$
\State Append $x$ to $\bx$
\State \textbf{return} $\bx$.
\end{algorithmic}
\end{algorithm}

\begin{lem} \label{lem:privatesampling} Let $A(D, \bx_0)$ be Algorithm \ref{alg:privatesampling}. For each $\bx_0 \in \CX^*$ the mechanism $M : D \mapsto A(D, \bx_0)$ satisfies $\rho$-zCDP, where $\rho = \frac12(\frac{c}{s\tau})^2$.\end{lem}
\begin{proof} Consider a function $f : \CD \rightarrow \bbR^v$ with sensitivity $\Delta$. By an analysis of the exponential mechanism due to \citet{cesar2021bounding},\footnote{See also \citet{DPorg-exponential-mechanism-bounded-range}.} choosing a token according to the distribution $\softmax(\frac{\varepsilon}{2\Delta})$ satisfies $\frac18\varepsilon^2$-zCDP. Observe that mechanism $M$ is the exponential mechanism with $f = \frac1\tau \ell$, which by Lemma \ref{lem:sensitivity} has sensitivity $\frac{c}{s\tau}$.
\end{proof}

\begin{algorithm}
\setstretch{1.15}
\caption{\label{alg:abovethresholdsampling} Public token generation}
\begin{algorithmic}[1]
\Statex {\bf Input:} Sensitive prompt dataset $D$, initial token sequence $\bx_0$
\Statex {\bf Output:} Token sequence $\bx \in \CX^*$
\State $\bx \gets \bx_0$
\State $\htheta \gets \theta + \textrm{Laplace}(\sigma)$
\While{True}
\State $Z \gets \{\logits(\bp\bx) : \bp \in D\}$
\State $\bz_{\public} \gets \logits(\bp_{\public}\bx)$
\State $\hd \gets \mathrm{d}(Z, \bz_{\public}) + \textrm{Laplace}(2\sigma)$
\If{$\hd \geq \htheta$}
\State Break
\Else
\State $x \sim \softmax(\bz_{\public}/\tau_{\public})$
\State Append $x$ to $\bx$
\EndIf
\EndWhile
\State \textbf{return} $\bx$.
\end{algorithmic}
\end{algorithm}

\begin{lem} \label{lem:abovethresholdsampling} Let $A(D, \bx_0)$ be Algorithm \ref{alg:abovethresholdsampling}. For each $\bx_0 \in \CX^*$ the mechanism $M : D \mapsto A(D, \bx_0)$ satisfies $\rho$-zCDP, where $\rho = \frac{2}{(s\sigma)^2}$.
\end{lem}
\begin{proof} Observe that mechanism $M$ is an instance of the AboveThrehold mechanism \citep{dwork2009complexity}, which accepts a private dataset, a threshold, and a sequence of queries as input. In each iteration, the AboveThreshold mechanism applies the next query in the sequence to the dataset and compares it to a noisy threshold, and returns the index of the first query that exceeds the threshold. The queries can be chosen adaptively and adversarially. In mechanism $M$, each query is specified by a token sequence $\bx$, and the index of the first query that exceeds the threshold is determined by the length of the returned token sequence. Furthermore, by Lemma \ref{lem:sensitivity} each query has sensitivity $\frac1s$. Thus by the analysis due to \citet{dwork2009complexity}, mechanism $M$ satisfies $\frac{2}{s\sigma}$-differential privacy, which by \citet{bun2016concentrated} implies that mechanism $M$ satisfies $\frac{2}{(s\sigma)^2}$-zCDP.
\end{proof}

\subsection{Putting it all together}
\label{sec:together}

Consider a sequence of iterations of the inner loop of Algorithm \ref{alg:main} in which the value of $t$ (the private token counter) is constant. Observe that the operation of Algorithm \ref{alg:main} during these iterations is equivalent to the sequential composition of Algorithms \ref{alg:privatesampling} and \ref{alg:abovethresholdsampling}, since these iterations generate zero or more public tokens followed by a private token.\footnote{The special treatment of the \texttt{<eos>} token complicates this story a little, but we can always assume that the LLM ignores any tokens before the last \texttt{<eos>} token.} Moreover, there are at most $r$ such sequences of iterations, since $r$ is an upper bound on the private token counter for any batch. By Lemmas \ref{lem:sequential}, \ref{lem:privatesampling} and \ref{lem:abovethresholdsampling} we have that Algorithm \ref{alg:main} applied to a single batch satisfies $\rho$-zCDP (where $\rho$ is specified in the statement of Theorem \ref{thm:main}). And therefore by Assumption \ref{assum:batch} and Lemma \ref{lem:parallel} we have that Algorithm \ref{alg:main} applied to the whole dataset satisfies $\rho$-zCDP. It remains to convert this zCDP guarantee to an $(\eps, \delta)$-differential privacy guarantee, which we do two different ways using two existing results: Corollary 13 due to \citet{canonne2020discrete} and Lemma 3.5 due to \citet{bun2016concentrated}.

\section{Privacy-equivalent Gaussian noise}
\label{appendix:noise}
Given the average token distribution $\bbap$ in a batch, \citet{tang2024privacypreserving} protect the privacy of $\bbap$ by using the Gaussian mechanism, which achieves $(\varepsilon, \delta)$-differential privacy with $\varepsilon = \frac{\sqrt{2\log(1.25/\delta)}}{s\sigma}$, where $s$ is the batch size and $\sigma$ is the standard deviation of the noise added to each probability in $\bbap$. On the other hand, we use the exponential mechanism to protect the privacy of a sample drawn from $\bbap$, which achieves $\eps$-differential privacy with $\eps = \frac{2c}{s\tau}$, where $c$ is the maximum absolute value of any log-probability in the batch and $\tau$ is the sampling temperature. 

Empirically, we obtained good synthetic data quality with $s = 250$, $\tau = 2$, $c = 10$ and $\delta = 10^{-6}$.

Setting the $\varepsilon$ values equal to each other yields $\sigma = \frac{\tau\sqrt{\log(1.25/\delta)}}{\sqrt{2}c}$, which is the noise level needed for the two mechanisms to have comparable privacy guarantees (setting aside that $\delta > 0$, an omission that only favors the Gaussian mechanism). Plugging in the above parameters yields $\sigma \approx 0.53$.

The analysis in Theorem 8 of \citet{balle2018improving} does not admit a closed-form solution. Instead, we binary search for a solution to:

$$\Phi\left(\frac{\Delta}{2\sigma} - \frac{\varepsilon\sigma}{\Delta}\right) - \exp(\varepsilon)\Phi\left(-\frac{\Delta}{2\sigma} - \frac{\varepsilon \sigma}{\Delta}\right) \leq \delta$$

where $\Phi$ is the Gaussian cdf, $\varepsilon = \frac{2c}{s\tau}$, $\delta = 10^{-6}$, and $\Delta$ is the L2 sensitivity of a vector computed as the average of $s$ user-provided probability vectors, namely $\Delta = 1/s$. This procedure yields $\sigma \approx 0.34$.

Finally, equating the zCDP loss for the exponential mechanism given by $\frac{\varepsilon^2}{8} = \frac{c^2}{2 s^2 \tau^2}$ (\citet{cesar2021bounding}) to that of the Gaussian mechanism given by $\frac{1}{ 2s^2\sigma^2}$ (\citet{bun2016concentrated}), yields $\sigma = 0.2$.

\section{Experiment details}\label{appendix:experiments}

\subsection{Hyperparameter tuning}
There are a significant amount of hyperparameters associated with our approach. See Table \ref{tab:hparams} for a list of the main ones and the values they take. In this section we describe the hyperparameter evaluation procedure, as well as the rationale for our decisions on what hyperparameter settings to couple together or that we altogether avoid running.

\paragraph{Hyperparameter evaluation procedure.} For fine-tuning experiments, we set aside a real validation set consisting of 10\% the real train set. We choose dataset generation parameters based on which resulting dataset induces the the best classifier on this real validation set. However, the process of tuning the classifier itself on synthetic data (choosing the best learning rate and checkpoint) does not use real data --  we do that tuning with synthetic data. This is because the output of our method is a dataset, and its usefulness to train a model includes how well subsets of it can be used for downstream task hyperparameter selection. After identifying the best synthetic dataset in this manner, we run the tuning process based on synthetic data only and report accuracy of the resultant classifier on the real test set.

\paragraph{Hyperparameter choices.} Based on initial experiments, we found that setting $c=10$ and $\tau = 2$ produced well formed text, so we fix $c=10$ and try a low temperature ($\tau=1.5$) and a high temperature ($\tau=2.25$) setting. At $\tau=2.25$, we observed text degeneration. This is due to the combination of Gemma's large vocabulary (256K) and clipping, which raises the ``probability floor'' of nonsense tokens. So for $\tau=2.25$ settings only, we follow \citet{tang2024privacypreserving} and reduce the vocabulary to the public prediction's top 1024. We emphasize that (1) we do not do this for any of the other settings of $\tau$, and (2) use a larger value than prior work (they use top 100).

Keeping other parameters fixed and increasing the batch size $s$ decreases $\varepsilon$. At the same time, it raises the amount of compute spent to decode a single example.\footnote{The way we interpret this is that $s$ is a compute multiplier that broadens the search space to include better utility configurations in the low $\varepsilon$ regime. This is analagous to the role of the noise multiplier $\sigma$ in DP-SGD, where the best results at low $\varepsilon$ come from taking more steps at higher noise levels.} Hence our approach for selecting the batch size is based on the following: given a target epsilon and dataset, choose $s$ large enough so that we can hit at least 1K examples at the low temperature setting $\tau=1.5$. When targeting a large $\varepsilon$, choosing large $s$ results in too many tokens to decode at too high of a cost per token.

For the sparse vector hyperparameters, we consider the following paired $(\theta,\sigma)$ settings: $\{(-\infty, -),$ $(0.3, 0.1),$ $(0.5, 0.2), (0.7, 0.3)\}$. The first setting corresponds to no use of the SVT, the next 3 represent different privacy levels per token: moving to the right uses noisier queries (less privacy budget) and more often uses the free public tokens. For large datasets and target $\varepsilon$, we do not run the high privacy settings (too much compute to finish), and for smaller datasets and smaller $\varepsilon$ we omit the settings that do not produce at least 1K examples.

\begin{table}[!ht]
    \small
    \centering
    \begin{tabular}{lll}
    \toprule
    $\alpha$ & Description & Values \\
    \midrule
    \multirow{2}{*}{$s$} & \multirow{2}{*}{batch size} & 127, 255, 511, \\
    & & 1023, 1535, 2047 \\
    \midrule
    $c$ & logits clip bound & 10 \\
    $\tau$ & temperature & 1.5, 2, 2.25 \\
    \midrule
    $\theta$ & SVT threshold  & $-\infty$, 0.3, 0.5, 0.7 \\
    $\sigma$ & SVT noise level  & $-$, 0.1, 0.2, 0.2 \\
    \midrule
    $\tau_\text{public}$ & public temperature & 1.5 \\
    \bottomrule
    \end{tabular}
    \caption{\small Values for hyperparameters explored in this work.}
    \label{tab:hparams}
\end{table}

\subsection{Prompts used}

We report the prompts used for our experiments. Generally, we use the same prompt for private and public predictions, with "\texttt{<text of xxx>}" in the public prompt replaced with an actual private example in the private prompt. The exception is for WikiMoviesJSON (Figures \ref{fig:wiki-movies-json-private-prompt} and \ref{fig:wiki-movies-json-public-prompt}), where the public prompt contains a schema description in place of the example.

\begin{figure*}[!ht]
\small
\centering

\begin{minipage}{0.9\textwidth}
\begin{center}
\begin{tcolorbox}
\begin{lstlisting}[language=Python]
# [User]
Here are texts with News Type: Business.

Text: <text of News Type: Business>

Please give me another one.

# [Assistant]
Text:
\end{lstlisting}
\end{tcolorbox}
\end{center}

\end{minipage}
\caption{\small Generation prompt for AGNews.}
\end{figure*}

\begin{figure*}[!t]
\small
\centering

\begin{minipage}{0.9\textwidth}
\begin{center}
\begin{tcolorbox}
\begin{lstlisting}[language=Python]
# [User]
Here are questions with Answer Type: Entity.

```
Text: <question of Answer Type: Entity>
```

Please give me another one.

# [Assistant]
```
Question:
\end{lstlisting}
\end{tcolorbox}
\end{center}

\end{minipage}
\caption{\small Generation prompt for TREC.}
\end{figure*}

\begin{figure*}[!t]
\small
\centering

\begin{minipage}{0.9\textwidth}
\begin{center}
\begin{tcolorbox}
\begin{lstlisting}[language=Python]
# [User]
Here are entries of Category: School.

Entry: <entry of Category: School>

Please give me another one.

# [Assistant]
Entry:
\end{lstlisting}
\end{tcolorbox}
\end{center}

\end{minipage}
\caption{\small Generation prompt for DBPedia.}
\end{figure*}

\begin{figure*}[!t]
\small
\centering

\begin{minipage}{0.9\textwidth}
\begin{center}
\begin{tcolorbox}
\begin{lstlisting}[language=Python]
# [User]
Give me text about a film and the extracted Phrase about its Director.

Phrase: "josh trank"
Text: "<text containing phrase "josh trank">"

Please give me another Phrase and Text: "josh trank". IMPORTANT: The exact Director phrase "josh trank" must be mentioned in Text.

# [Assistant]
Phrase: "josh trank"
Text: "
\end{lstlisting}
\end{tcolorbox}
\end{center}

\end{minipage}
\caption{\small Generation prompt for MIT-D.}
\end{figure*}

\begin{figure*}[!t]
\small
\centering

\begin{minipage}{0.9\textwidth}
\begin{center}
\begin{tcolorbox}
\begin{lstlisting}[language=Python]
# [User]
Give me text about a film and the extracted Phrase about its Genre.

Phrase "comedy"
Text: "<text containing phrase "comedy">"

Please give me another Phrase and Text. IMPORTANT: The exact Genre phrase "comedy" must be mentioned in Text.

# [Assistant]
Phrase: "comedy"
Text: "
\end{lstlisting}
\end{tcolorbox}
\end{center}

\end{minipage}
\caption{\small Generation prompt for MIT-G.}
\end{figure*}

\begin{figure*}[!t]
\small
\centering

\begin{minipage}{0.9\textwidth}
\begin{center}
\begin{tcolorbox}
\begin{lstlisting}[language=Python]
# [User]
Here are texts with Sentiment: Negative.

Text: <text of Sentiment: Negative>

Please give me another one.

# [Assistant]
Text: 
\end{lstlisting}
\end{tcolorbox}
\end{center}

\end{minipage}
\caption{\small Generation prompt for IMDB.}
\end{figure*}

\begin{figure*}[!t]
\small
\centering

\begin{minipage}{0.9\textwidth}
\begin{center}
\begin{tcolorbox}
\begin{lstlisting}[language=Python]
# [User]
Here are texts with Sentiment: Negative.

Text: <text of Sentiment: Negative>

Please give me another one.

# [Assistant]
Text: 
\end{lstlisting}
\end{tcolorbox}
\end{center}

\end{minipage}
\caption{\small Generation prompt for Yelp.}
\end{figure*}

\begin{figure*}[!t]
\small
\centering
\begin{minipage}{0.9\textwidth}
\begin{center}
\begin{tcolorbox}
\begin{lstlisting}[language=Python]
# [User]
Here is a JSON record:
```
{
  "title": "$50,000 Reward",
  "year": 1924,
  "cast": [
    "Ken Maynard",
    "Esther Ralston"
  ],
  "genres": [
    "Western",
    "Silent"
  ],
  "href": "$50,000_Reward",
  "extract": "$50,000 Reward is a 1924 American silent Western film directed by Clifford S. Elfelt and starring Ken Maynard, Esther Ralston and Bert Lindley."
}
```
Please give me another JSON record that complies with the above schema.

# [Assistant]
```
{

\end{lstlisting}
\end{tcolorbox}
\end{center}
\end{minipage}
\caption{\small Private generation prompt for WikiMoviesJSON.}
\label{fig:wiki-movies-json-private-prompt}
\end{figure*}

\begin{figure*}[!t]
\small
\centering
\begin{minipage}{0.9\textwidth}
\begin{center}
\begin{tcolorbox}
\begin{lstlisting}[language=Python]
# [User]
Here is the schema for a JSON record:
Schema:
```
{
  "title": "str",
  "year": int,
  "cast": [ # list of str
    "str1", # 0 or more total entries
  ],
  "genres": [ # list of str
    "str1", # 0 or more total entries
  ]
  "href": "str", # URL slug, e.g.: Link_to_Page
  "extract": "str"
}
```
Please give me another JSON record that complies with the above schema.

# [Assistant]
```
{
\end{lstlisting}
\end{tcolorbox}
\end{center}
\end{minipage}
\caption{\small Public generation prompt for WikiMoviesJSON.}
\label{fig:wiki-movies-json-public-prompt}
\end{figure*}

\section{Artifacts}
Tables \ref{tab:datasets} and \ref{tab:models} list all artifacts we use in this work. AGNews, TREC, DBPedia, MIT-G, MIT-D, IMDB, and Yelp are all standard academic datasets permissible for research use; we cite their original publications when introduced. WikiMoviesJSON is scraped from Wikipedia data, courtesy of \citep{rust2024wikipediamoviedata}; their work is covered by an MIT license. Wikipedia content is licensed under the Creative Commons Attribution-ShareAlike 4.0 International License (CC BY-SA) and the GNU Free Documentation License (GFDL).

We use open-source models BERT-Base, released by \citep{turc2019}, and Gemma. Our use of Gemma for academic purposes is in accordance of the Gemma terms of use: \url{https://ai.google.dev/gemma/terms}. GPT-3 is accessible for academic purposes under OpenAI's terms of use, which supports educational and research activities. LaMDA 8B is not publically available, but we received sufficient authorization to use it for the academic purposes of this paper.

\section{Compute budget}
Our main experiments for synthetic data generation run on Gemma 1.1 2B IT. A run of synthetic data generation takes between 8-48 accelerator hours. Including exploratory runs and hyperparameter search, the total compute budget for this project is roughly 14,000 accelerator hours.

\end{document}